%% file: main_iclr2027_conference.tex
\documentclass{article} 
\usepackage{iclr2027_conference,times}

\input{math_commands.tex}

\usepackage{hyperref}
\usepackage{url}
\usepackage{booktabs}
\usepackage{graphicx}
\usepackage{subcaption}
\usepackage{algorithm}
\usepackage{algpseudocode}
\usepackage{wrapfig}
\usepackage{float}
\usepackage[table]{xcolor}
\usepackage{multirow}
\newcommand{\dd}{\mathrm{d}}

\newcommand{\norm}[1]{\lVert #1\rVert}

\usepackage{amsthm}

\theoremstyle{plain}
\newtheorem{theorem}{Theorem}[section]

\theoremstyle{definition}

\theoremstyle{remark}
\newtheorem{remark}[theorem]{Remark}

\usepackage[most]{tcolorbox}
\tcbuselibrary{theorems}

\newtcbtheorem[number within=section]
  {boxedtheorem}{Theorem}%
  {
    enhanced,
    breakable,
    colback=white,
    colframe=black,
    boxrule=0.8pt,
    arc=1mm,
    fonttitle=\bfseries
  }{bthm}

\newtcbtheorem[number within=section]
  {boxedproposition}{Proposition}%
  {
    enhanced,
    breakable,
    colback=white,
    colframe=black,
    boxrule=0.8pt,
    arc=1mm,
    fonttitle=\bfseries
  }{bprop}

\newtcbtheorem[number within=section]
  {boxeddefinition}{Definition}%
  {
    enhanced,
    breakable,
    colback=white,
    colframe=black,
    boxrule=0.8pt,
    arc=1mm,
    fonttitle=\bfseries
  }{bdef}

\title{Geometry-Aware Time Reparameterization for Flow-Map Distillation}   

\author{
F\'elix Dedek \& Makoto Yamada\\
Machine Learning and Data Science Unit\\
Okinawa Institute of Science and Technology Graduate University\\
Okinawa, Japan
}

\iclrfinalcopy 
\begin{document}

\maketitle
\fancyhead{}
\renewcommand{\headrulewidth}{0pt}

\input{abstract}

\input{introduction}

\input{relatedwork}

\input{Background}

\input{method}

\input{Experiments}

\input{conclusion}

\bibliography{iclr2027_conference}
\bibliographystyle{iclr2027_conference}

\appendix
\input{appendix_proofs}

\input{Additional_experimental_details}

\end{document}

%% file: math_commands.tex
\usepackage{amsmath,amsfonts,bm}

\def\eqref#1{equation~\ref{#1}}

\def\1{\bm{1}}

\def\rt{{\textnormal{t}}}

\def\rva{{\mathbf{a}}}
\def\rvb{{\mathbf{b}}}

\def\rvx{{\mathbf{x}}}
\def\rvy{{\mathbf{y}}}

\def\vtheta{{\bm{\theta}}}

\def\vx{{\bm{x}}}
\def\vy{{\bm{y}}}

\def\mI{{\bm{I}}}

\DeclareMathAlphabet{\mathsfit}{\encodingdefault}{\sfdefault}{m}{sl}
\SetMathAlphabet{\mathsfit}{bold}{\encodingdefault}{\sfdefault}{bx}{n}

\newcommand{\E}{\mathbb{E}}



%% file: abstract.tex
\begin{abstract}
Flow-map distillation enables one- and few-step generation by learning
finite-time transitions of a pretrained generative ODE.
We investigate whether changing the teacher's time parameterization
can make these transitions easier to learn.
Motivated by the hypothesis that trajectory segments with large
normal acceleration are harder to distill, we propose a geometry-aware
time reparameterization that allocates more student time to these
regions while preserving the teacher's geometric paths and terminal
distribution.
We derive a shared clock that equalizes a population
normal-acceleration statistic under suitable assumptions, and
construct a practical approximation from robust, regularized
estimates across teacher trajectories.
We incorporate this clock into Lagrangian flow-map distillation,
using the transformed time coordinate to condition the student.
The clock is estimated once before distillation and requires
neither teacher retraining nor additional student parameters or
inference-time network evaluations.
Experiments on synthetic data, CIFAR-10, and CelebA-64 show improved
sample quality over identity-time distillation at matched inference
budgets, including improvements in one- and two-step image generation.
The gains in one-step generation, where no intermediate sampling
times can be adjusted, highlight the benefits of time
reparameterization during distillation.
\end{abstract}

%% file: introduction.tex
\section{Introduction}
Diffusion and score-based generative models have emerged as a leading paradigm for learning complex data distributions
\citep{sohldickstein2015deep,ho2020denoising,song2021score},
enabling high-fidelity image synthesis and flexible conditioning at scale
\citep{dhariwal2021diffusion,rombach2022highresolution}.
These models learn to reverse a gradual noising process, with generation described by reverse-time stochastic dynamics or an associated probability-flow ordinary differential equation (ODE).
Their success, however, comes with substantial inference cost: producing a sample typically requires many sequential neural-network evaluations.
Reducing this cost while preserving sample quality remains a central challenge, motivating extensive work on deterministic sampling and specialized numerical solvers
\citep{song2021denoising,lu2022dpmsolver}.

Flow matching (FM) \citep{lipman2023flowmatching} provides a flexible framework for learning continuous-time transport through a simulation-free regression objective.
Together with closely related rectified-flow and stochastic-interpolant formulations
\citep{liu2023rectifiedflow,albergo2023stochasticinterpolants},
it learns a time-dependent velocity field that transports a simple source distribution to the data distribution, using analytically accessible conditional paths to construct training targets.
This viewpoint accommodates diffusion probability paths as well as simpler alternatives, such as linear interpolation between noise and data, offering greater flexibility in the design of the generative dynamics.
Importantly, however, straight conditional interpolation paths do not generally imply straight trajectories of the resulting marginal velocity field \citep{frans2025shortcut}.
Sampling therefore still requires numerical ODE integration, and using too few steps can introduce substantial discretization error.

One- and few-step generative models address this limitation by learning finite-time transport directly, either through distillation from a pretrained model \citep{salimans2022progressive} or through direct training from data.
Consistency models \citep{song2023consistency} learn to map points along a probability-flow trajectory to a common endpoint, while consistency trajectory models \citep{kim2024consistencytrajectory} and flow map matching (FMM) \citep{boffi2025flowmap} learn two-time maps connecting intermediate states along the dynamics.
Related formulations include Shortcut Models \citep{frans2025shortcut}, which learn step-size-conditioned updates by enforcing consistency between one large step and two smaller steps, and MeanFlow \citep{geng2025meanflow}, which learns average velocities through a differential identity relating them to instantaneous velocities.
By learning finite-time updates, these approaches replace many evaluations of an instantaneous velocity field with one or a few evaluations of a learned map.
Their effectiveness consequently depends on how accurately this map can be learned.

We investigate whether changing a pretrained teacher's time parameterization can make its flow map easier to learn. Motivated by the hypothesis that segments with large normal acceleration are harder to distill, we construct a shared clock that allocates more student time to these regions while preserving the teacher's geometric paths and terminal distribution. The clock is estimated once from robust, regularized normal-acceleration statistics. We then train the student against the reparameterized velocity $\widetilde b_\tau(x)=\phi'(\tau)b_{\phi(\tau)}(x)$ using Lagrangian flow-map distillation. This requires neither teacher retraining nor changes to the student architecture or additional inference-time network evaluations.

Our main contributions are: 
\begin{itemize} 
\item \textbf{A normal-acceleration-based clock.} We derive a shared time transformation that equalizes a population normal-acceleration statistic under suitable assumptions, and estimate it robustly from teacher trajectories.

\item \textbf{Geometry-aware flow-map distillation.} We incorporate this time coordinate into Lagrangian flow-map distillation, training a single student directly against the reparameterized teacher dynamics. The time map is estimated once before distillation and requires neither teacher retraining nor changes to the student architecture, with no additional inference-time network evaluations. 
\item \textbf{Improved few-step generation.} Through synthetic and image-generation experiments, we demonstrate improved sample quality over identity-time distillation at matched inference budgets. In particular, we observe meaningful improvements to one-step generation, where no intermediate sampling times can be adjusted, highlighting the role of time reparameterization in the distillation problem itself. \end{itemize}

%% file: relatedwork.tex
\section{Related work}
\vspace{-.1in}
\paragraph{Flow-map learning.}

Progressive distillation \citep{salimans2022progressive} compresses
successive teacher updates, while consistency models and consistency
trajectory models learn endpoint predictions or transitions between
intermediate noise levels
\citep{song2023consistency,kim2024consistencytrajectory}.
Flow map matching \citep{boffi2025flowmap} provides a mathematical
framework for these constructions, with objectives for both teacher
distillation and direct training.
Subsequent work develops self-distillation formulations
\citep{boffi2025buildconsistencymodellearning} and practical continuous-time
distillation objectives \citep{sabour2025alignyourflow}.
Shortcut Models and MeanFlow offer complementary formulations based on
compositional consistency and average velocities
\citep{frans2025shortcut,geng2025meanflow}.
Recent extensions include parallel prediction of interval velocities
\citep{shaul2026paralleldecoding} and distribution-based distillation
using average-velocity information \citep{zhao2026meanflowdistillation}.
Our work builds on Lagrangian flow-map distillation and investigates
the time coordinate in which the teacher's finite-time transport is
learned.

\vspace{-.1in}
\paragraph{Trajectory geometry and coupling.}
Several approaches improve generative dynamics by simplifying their
trajectories.
Rectified flow and reflow encourage straighter transport
\citep{liu2023rectifiedflow}, while minibatch optimal transport modifies
source--target couplings to obtain simpler dynamics
\citep{tong2024improving}.
\citet{lee2023minimizing} learn couplings that reduce trajectory
curvature and also investigate the benefits of straighter teachers
for distillation.
Isokinetic Flow Matching introduces a lookahead velocity-consistency
regularizer motivated by reducing material acceleration
\citep{khan2026isokineticflowmatchingpathwise}.
These methods modify the learned velocity field or its training
coupling.
Our transformation instead retains the geometric trajectories of a
pretrained teacher and adjusts their traversal speed before
flow-map distillation.

\vspace{-.1in}
\paragraph{Sampling schedules and trajectory diagnostics.}
Sampling efficiency also depends on the placement of numerical
integration steps.
Align Your Steps optimizes schedules through a stochastic-calculus-based
discrepancy criterion
\citep{sabour2024alignstepsoptimizingsampling}, while Geometry-Inspired
Time Scheduling uses trajectory information and dynamic programming
to select integration times
\citep{chen2024trajectoryregularityodebaseddiffusion}.
Entropic Time Schedulers allocate sampling points according to
conditional-entropy progress \citep{stancevic2025entropic}.
Particularly close to our construction, Sharpen Your Flow estimates
the full material-acceleration magnitude of a pretrained flow and
derives a square-root sampling density implemented through
inverse-CDF scheduling
\citep{gupta2026sharpenflowsharpnessawaresampling}.
Our method uses the normal component of acceleration and applies
the resulting clock to the teacher dynamics during student training.
This distinction is particularly relevant to one-step flow-map
generation, where no intermediate inference times can be adjusted.

\vspace{-.1in}
\paragraph{Adaptive allocation of training effort.}
Temporal difficulty has also motivated adaptive training procedures.
Curvature-Adaptive Consistency Flow Matching uses consistency-error
feedback and reinforcement learning to allocate training across
temporal stages \citep{tian2026curvatureadaptive}.
Its boundary-focused allocation provides qualitative evidence of
nonuniform distillation difficulty, although its diagnostic differs
from normal acceleration.
Complexity-Balanced Diffusion Splitting uses acceleration or spatial
regularity estimates to partition time among specialized networks
\citep{issachar2026complexitybalanced}.
Our approach trains a single student using a fixed time transformation
estimated from the pretrained teacher.

\vspace{-.1in}
\paragraph{Time reparameterization and learnability.}
Time parameterization can affect learning as well as numerical
integration.
\citet{tsimpos2025optimal} optimize the schedule of a fixed transport
interpolation to control the spatial Lipschitz constant of its
velocity field.
\citet{chen2025lipschitz} develop Lipschitz-guided interpolation
schedules and transfer formulas for applying them to pretrained
models.
Other approaches prescribe conditional velocity profiles
\citep{bondar2026velocityscheduled} or construct an inverse-CDF
interpolation schedule from regression losses measured during a
pilot training run \citep{tania2026difficultycalibrated}.
Time transformations also appear in Flow Map Language Models,
where a decoding-error-based clock guides temporal sampling and
flow-map learning \citep{lee2026flowmaplanguage}.
These works establish the relevance of temporal coordinates to
generative learning.
We investigate a shared clock derived from the normal acceleration
of a frozen teacher, and train the student directly against the
reparameterized teacher velocity in the Lagrangian distillation
objective.

%% file: Background.tex
\section{Background}
In this section, we first formulate the flow matching problem and then introduce flow-map distillation methods. 

\subsection{Flow Matching}

Let $(\check{\rvx}_0,\check{\rvx}_1)\sim\pi\in\Pi(\rho_0,\rho_1)$
denote a pair of samples from a coupling of the source and target
distributions.
Following the stochastic-interpolant formulation
\citep{albergo2023stochasticinterpolants}, we define intermediate
states through a time-differentiable interpolant:
\[
    I_t(\check{\rvx}_0,\check{\rvx}_1)
    :=
    \alpha_t\check{\rvx}_0+\beta_t\check{\rvx}_1,
\]
with boundary conditions
\[
    I_0(\check{\rvx}_0,\check{\rvx}_1)=\check{\rvx}_0,
    \qquad
    I_1(\check{\rvx}_0,\check{\rvx}_1)=\check{\rvx}_1.
\]
For brevity, we write
$I_t:=I_t(\check{\rvx}_0,\check{\rvx}_1)$,
where $(\check{\rvx}_0,\check{\rvx}_1)\sim\pi$.

Flow matching \citep{lipman2023flowmatching} learns a time-dependent
velocity field by regressing the interpolant velocity:
\[
    \mathcal L_{\mathrm{FM}}(\vtheta)
    =
    \E_{t,(\check{\rvx}_0,\check{\rvx}_1)\sim\pi}
    \left[
        \norm{
            b^{\vtheta}_t(I_t)
            -\dot I_t
        }^2
    \right],
\]
where $t\sim\mathcal U([0,1])$ independently of the endpoint pair
and $\dot I_t=\frac{\dd}{\dd t}I_t$.
This loss is minimized by
\[
    b_t^\star(\vx)
    =
    \E\left[
        \dot I_t
        \mid I_t=\vx
    \right].
\]
Under suitable regularity assumptions, the ODE driven by $b_t^\star$ and initialized from $\rho_0$ reproduces the interpolant marginals.
In all our experiments, we use linear interpolation,
\[
    I_t(\check{\rvx}_0,\check{\rvx}_1)
    =
    (1-t)\check{\rvx}_0+t\check{\rvx}_1,
    \qquad
    \dot I_t(\check{\rvx}_0,\check{\rvx}_1)
    =
    \check{\rvx}_1-\check{\rvx}_0.
\]

We denote the trained teacher field by $b_t$ and reserve $\rvx_t$
for its ODE trajectories:
\[
    \dot{\rvx}_t=b_t(\rvx_t),
    \qquad \rvx_0\sim\rho_0.
\]
These trajectories generally differ from individual interpolation
paths.
Generation requires numerical ODE integration, typically involving
many network evaluations and motivating few-step samplers and distillation methods.

\subsection{Flow-Map Distillation}

Consistency models \citep{song2023consistency} learn to map states
along a probability-flow trajectory to a common data endpoint.
Flow-map distillation \citep{boffi2025flowmap} learns two-time maps
connecting intermediate states of a pretrained teacher ODE.
Let $X_{s,t}$ denote its solution map from time $s$ to time $t$:
\begin{equation}
    X_{s,t}(\rvx_s)=\rvx_t.
    \label{eq:flow map condition}
\end{equation}

Approximating this map allows a student to replace many evaluations
of the teacher velocity $b_t$ with a single finite-time update.

Following \citet{boffi2025flowmap}, we parameterize the student as
\begin{equation}
\label{eq: flow map representation}
    \widehat X^\vtheta_{s,t}(\vx)
    =
    \vx+(t-s)v_\vtheta(\vx,s,t),
\end{equation}
where $v_\vtheta$ approximates the average velocity over
the interval $[s,t]$.
This parameterization enforces the identity condition
$\widehat X^\vtheta_{s,s}(\vx)=\vx$.

Assuming the teacher ODE has unique solutions, the flow-map relation
\eqref{eq:flow map condition} is equivalently characterized by the
Lagrangian equation
\[
    \partial_t X_{s,t}(\vx)
    =
    b_t\bigl(X_{s,t}(\vx)\bigr),
    \qquad
    X_{s,s}(\vx)=\vx.
\]
Since our parameterization enforces the identity condition,
Lagrangian flow-map distillation (LMD)
\citep{boffi2025flowmap} trains the student by minimizing
the residual of this equation.
We use the following stop-gradient formulation:
\begin{equation}
\label{eq:LMD loss}
    \mathcal L_{\mathrm{LMD}}(\vtheta)
    =
    \E_{s,t,(\check{\rvx}_0,\check{\rvx}_1)\sim\pi}
    \left[
        \left\|
            \partial_t \widehat X^\vtheta_{s,t}(I_s)
            -
            \mathrm{sg}\!\left(
                b_t\!\left(\widehat X^\vtheta_{s,t}(I_s)\right)
            \right)
        \right\|^2
    \right].
\end{equation}
The time-pair sampling distribution is specified in the experimental setup.

Following the stop-gradient treatment of Lagrangian distillation
discussed by \citet{boffi2025buildconsistencymodellearning}, we apply
$\mathrm{sg}(\cdot)$ to the teacher velocity evaluated at the
student output.
This avoids backpropagation through the teacher's spatial Jacobian,
while the teacher parameters remain fixed throughout training.
We also found this treatment important for training stability
in our experiments; further details are provided in the appendix.

\paragraph{Sampling from the flow map}
Given a partition $0=t_0<\cdots<t_K=1$, we generate samples
by composing the learned maps:
\[
    \widehat{\rvx}_{t_{k+1}}
    =
    \widehat X^\vtheta_{t_k,t_{k+1}}
    (\widehat{\rvx}_{t_k}),
    \qquad k=0,\ldots,K-1,
    \qquad
    \widehat{\rvx}_0\sim\rho_0.
\]
The terminal state $\widehat{\rvx}_1$ is the generated sample.
Each update requires one student-network evaluation, allowing
the sampling budget $K$ to be chosen after training.
In particular, one-step generation applies
$\widehat X^\vtheta_{0,1}$ directly to a source sample.


%% file: method.tex
\section{Proposed Method}
We first motivate normal acceleration as a proxy for flow-map learning difficulty. We then construct a time transformation that
redistributes this quantity across time and incorporate the resulting clock into Lagrangian distillation.

\begin{wrapfigure}[25]{r}{0.5\textwidth}
    \centering
    \includegraphics[width=0.5\textwidth]{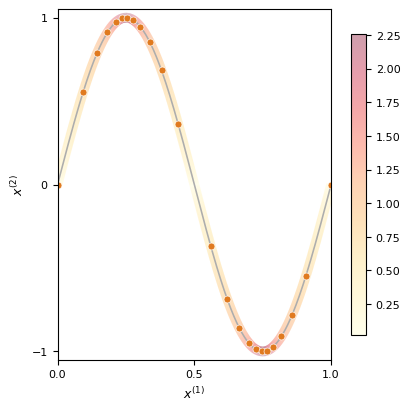}
    \caption{Time warping represented on the 2d trajectory $\vx_t=(t,\sin(2\pi t))$. The orange dots represent the reparametrized trajectory $\vx_{\phi(\tau_K)} $ with uniform time steps $\tau_k=k/(N+1)$. The colorscale represents the density $q=\psi\prime$ of $\phi(\rt)$ with $\rt\sim\mathcal{U}[0,1]$.}
    \label{fig:placeholder}
\end{wrapfigure}

\subsection{Normal Acceleration As A Proxy For Distillation Difficulty}
Our working hypothesis is that trajectory segments with large normal acceleration are harder to represent through finite-time flow maps. To motivate this, consider the exact average velocity along a teacher trajectory: \[ \frac{\rvx_t-\rvx_s}{t-s} = \frac{1}{t-s}\int_s^t b_u(\rvx_u)\,\dd u, \qquad s<t. \] If the velocity is constant on $[s,t]$, this average equals $b_s(\rvx_s)$. More generally, if its direction is fixed, the average remains aligned with that direction, and only its magnitude varies with the interval. When the direction changes, both the magnitude and direction of the average can depend on the two endpoints. This motivates using normal acceleration, the component associated with directional changes, as a diagnostic of temporal learning difficulty. The resulting connection to distillation error is an empirical hypothesis, rather than a consequence of the flow-map identity.

\paragraph{Normal-acceleration diagnostic.}
For the teacher field $b_t$, define the material acceleration
\[
    a_t(\vx)
    :=
    \partial_t b_t(\vx)
    +D_{\vx}b_t(\vx)\,b_t(\vx).
\]
Along a teacher trajectory, $\ddot{\rvx}_t=a_t(\rvx_t)$.
Where $b_t(\vx)\neq0$, let
$\widehat b_t(\vx)=b_t(\vx)/\norm{b_t(\vx)}$ and decompose
\[
    a_t^\parallel(\vx)
    =
    \bigl(\widehat b_t(\vx)^\top a_t(\vx)\bigr)
    \widehat b_t(\vx),
    \qquad
    a_t^\perp(\vx)=a_t(\vx)-a_t^\parallel(\vx).
\]
The tangential component changes speed, while the normal
component changes velocity direction.
We define
\[
    n_t(\vx):=\norm{a_t^\perp(\vx)}.
\]
Geometrically, $n_t$ measures how strongly a trajectory bends at its current traversal speed: it equals the geometric curvature multiplied by the squared speed. Thus, slowing down reduces normal acceleration without altering the path taken by a trajectory. 
To summarize this diagnostic across teacher trajectories, we define the \emph{hardness profile} \[ h_p(t) := \left(\E\left[n_t(\rvx_t)^p\right]\right)^{1/p},
\qquad p>0, \] 
where the expectation is over teacher trajectories initialized
from $\rvx_0\sim\rho_0$.
We use this profile as a proxy for time-dependent distillation difficulty.

\subsection{Geometry-Aware Time Reparameterization} The hardness profile identifies where additional temporal resolution may be useful. We therefore seek a shared change of time that slows traversal through regions of high hardness and accelerates it where hardness is low. Such a transformation preserves the teacher's geometric paths and terminal distribution, while changing its normal acceleration. Our goal is to choose the new time coordinate so that the resulting hardness profile is constant. The following theorem gives an explicit construction.

\begin{theorem}[Equalizing the hardness profile] 
\label{thm: phi construction} 
Assume the teacher trajectories are $C^2$ with nonzero velocity, and $h_p$ is positive and continuously differentiable. Define 
\[ 
\psi(t) := \frac{\int_0^t \sqrt{h_p(u)}\,\dd u} {\int_0^1 \sqrt{h_p(u)}\,\dd u}, \qquad \phi:=\psi^{-1}. 
\] 
Then the reparameterized trajectories $\rvy_\tau:=\rvx_{\phi(\tau)}$ have constant hardness: 
\[ 
\widetilde h_p(\tau)
:=
\left(\E[\widetilde n_\tau(\rvy_\tau)^p]\right)^{1/p} = \left(\int_0^1 \sqrt{h_p(u)}\,\dd u\right)^2, \qquad \tau\in[0,1], 
\] 
where $\widetilde n_\tau$ denotes the normal-acceleration magnitude of the reparameterized dynamics.
\end{theorem}
\begin{proof}
    See Appendix~\ref{prf:main thm}
\end{proof}

\begin{remark}[Tempering the time transformation]
The strength of the transformation can be controlled by choosing
$\psi'(t)\propto(h_p(t)+\xi)^{\delta/2}$, where $\delta\geq0$
and $\xi\geq0$ is an optional numerical floor.
The choice $\delta=0$ recovers the identity clock,
$0<\delta<1$ softens the transformation, and $\delta>1$
increases its concentration on high-hardness regions.
The theorem corresponds to $\delta=1$ and $\xi=0$.
\end{remark}

\paragraph{Practical construction.}
We estimate the time map once before distillation using $M$
teacher trajectories evaluated on a fine grid
$t_k=k/N$, $k=0,\ldots,N$, with $N\geq100$.
We estimate material acceleration by finite differences of
the teacher velocities along each trajectory, then project
onto the normal direction to obtain $n_{k,m}$.
At each time, we discard the largest $5\%$ of these values
before computing the hardness profile.
This limits the influence of rare sharp turns, which can
produce particularly large normal accelerations (phenomenon which we mainly observed in 2d experiments).
We integrate the resulting clock density numerically and fit a monotonicity-preserving PCHIP interpolant to the inverse pairs $(\tau_k,t_k)$ \citep{doi:10.1137/0905021}. The resulting $C^1$ map has an analytically evaluable derivative.
A positive floor $\xi$ ensures a strictly increasing cumulative
clock; in our experiments, the estimated hardness already
remained bounded away from zero. Trimming, tempering, and numerical approximation mean that the practical clock need not exactly equalize the population profile.
Algorithm~\ref{alg:clock-construction} summarizes the procedure.

\begin{algorithm}[!htbp]
\caption{Offline construction of the shared time map}
\label{alg:clock-construction}
\begin{algorithmic}[1]
\Require Teacher $b_t$, source $\rho_0$, trajectory count $M$,
grid size $N$, aggregation $p>0$, tempering $\delta\geq0$,
floor $\xi\geq0$
\State Generate $M$ teacher trajectories-velocity pairs $(\rvx_{k,m},\rvb_{k,m})$ from
$\rvx_{0,m}\sim\rho_0$ on $t_k=k/N$, $k=0,\ldots,N$, $m=1,\dots,M$
\State 
$\rva_{k,m}\gets\operatorname{FD}_t[\rvb_{\cdot,m}]_k$
\State $\hat{\rvb}_{k,m}\gets
\rvb_{k,m}/\norm{\rvb_{k,m}}$;
$n_{k,m}\gets
\norm{\rva_{k,m}
-(\hat{\rvb}_{k,m}^{\top}\rva_{k,m})\hat{\rvb}_{k,m}}$
\For{$k=0,\ldots,N$}
    \State $\mathcal I_k\gets$ indices of the smallest
    $\lceil0.95M\rceil$ values of $\{n_{k,m}\}_{m=1}^M$
    \State $\widehat h_{p,k}\gets
    \left(
        \frac{1}{|\mathcal I_k|}
        \sum_{m\in\mathcal I_k}n_{k,m}^{p}
    \right)^{1/p}$;
    $w_k\gets(\widehat h_{p,k}+\xi)^{\delta/2}$
\EndFor
\State $S_0\gets0$;
$S_k\gets S_{k-1}
+\frac{t_k-t_{k-1}}{2}(w_{k-1}+w_k)$,
$k=1,\ldots,N$
\State $\tau_k\gets S_k/S_N$, $k=0,\ldots,N$
\State $\phi\gets
\operatorname{PCHIP}\bigl(\{(\tau_k,t_k)\}_{k=0}^N\bigr)$
\State \Return $\phi$ and its derivative $\phi'$
\end{algorithmic}
\end{algorithm}

\paragraph{Flow-map distillation in warped time}
Given the clock $\phi=\psi^{-1}$ (Theorem \ref{thm: phi construction}), we train a student to approximate
the reparameterized flow map
\[
    Y_{\alpha,\beta}
    :=
    X_{\phi(\alpha),\phi(\beta)},
    \qquad 0\leq\alpha\leq\beta\leq1.
\]
The trajectories $\rvy_\tau=\rvx_{\phi(\tau)}$ satisfy
\[
    \frac{\dd\rvy_\tau}{\dd\tau}
    =
    \widetilde b_\tau(\rvy_\tau),
    \qquad
    \widetilde b_\tau(\vx)
    :=
    \phi'(\tau)b_{\phi(\tau)}(\vx).
\]

For image experiments, we condition the student on the reparameterized times while retaining the physical-time interval as its displacement scale:
\begin{equation}
\label{eq:warped-map-parameterization}
    \widehat Y^\vtheta_{\alpha,\beta}(\vx)
    =
    \vx+
    \bigl(\phi(\beta)-\phi(\alpha)\bigr)
    v_\vtheta(\vx,\alpha,\beta).
\end{equation}

This choice separates temporal allocation from velocity scaling.
Empirically, the original teacher's velocity norm is approximately
constant over most of the time interval, whereas reparameterization
introduces substantial variation through the factor $\phi'$
(see Figure~\ref{fig:celeb metrics}).
The parameterization also enforces
$\widehat Y^\vtheta_{\alpha,\alpha}(\vx)=\vx$. For the toy experiments, we instead use displacement scale $\beta-\alpha$; the warped-time distillation objective below applies to either parameterization. Both scalings coincide under the identity clock.

We apply the Lagrangian distillation objective to the
reparameterized teacher:
\begin{equation}
\label{eq:warped-lmd}
    \mathcal L_{\mathrm{LMD}}^\phi(\vtheta)
    =
    \E_{\alpha,\beta,(\check{\rvx}_0,\check{\rvx}_1)\sim\pi}
    \left[
        \left\|
            \partial_\beta
            \widehat Y^\vtheta_{\alpha,\beta}(I_{\phi(\alpha)})
            -
            \mathrm{sg}\!\left(
                \widetilde b_\beta\!\left(
                    \widehat Y^\vtheta_{\alpha,\beta}
                    (I_{\phi(\alpha)})
                \right)
            \right)
        \right\|^2
    \right].
\end{equation}

For image experiments, we sample two independent uniform physical times and sort them to obtain $s\leq t$. The student receives $\alpha=\psi(s)$ and $\beta=\psi(t)$,
with input state $I_s=I_{\phi(\alpha)}$.
Thus, both clocks use the same physical-time pair distribution,
while $\phi$ changes the student's temporal coordinates.
Both the teacher and the clock remain fixed during training.
We additionally apply learned loss weighting, as detailed
in Appendix~\ref{app:image-experiments}.

We sample analogously by composing
$\widehat Y^\vtheta_{\tau_k,\tau_{k+1}}$
over a student-time partition $0=\tau_0<\cdots<\tau_K=1$,
starting from $\widehat{\rvy}_0\sim\rho_0$. The overhead that comes from computing the scaling factor $(\phi(\tau_{k+1})-\phi(\tau_{k}))$ is negligible compared with neural-network evaluation, and these differences can be precomputed once for a fixed partition.

%% file: Experiments.tex
\section{Experiments}

We first examine the temporal distribution of normal
acceleration and the effect of our clock on teacher
sampling.
We then evaluate whether training in the resulting
time coordinate improves flow-map distillation on
synthetic and image datasets.

\subsection{Temporal Geometry and Teacher Sampling}
\label{sec:geometry-and-sampling}

\begin{figure}
    \centering
    
    \begin{subfigure}{0.30\linewidth}
    \centering
        \includegraphics[width=\linewidth]{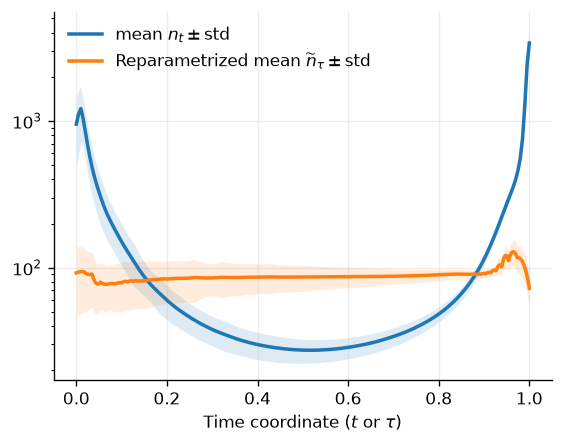}
        
    \end{subfigure}
    \hfill
    \begin{subfigure}{0.30\linewidth}
    \centering
        \includegraphics[width=0.85\linewidth]{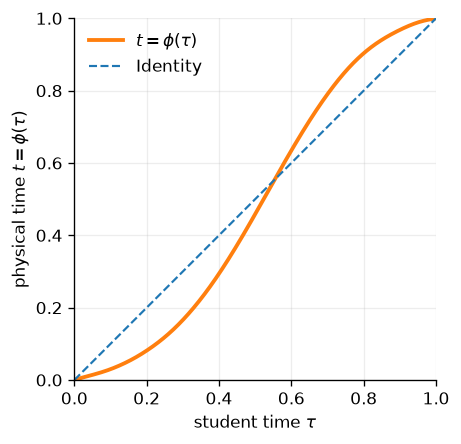}
        
    \end{subfigure}
    \hfill
    \begin{subfigure}{0.30\linewidth}
    \centering
        \includegraphics[width=\linewidth]{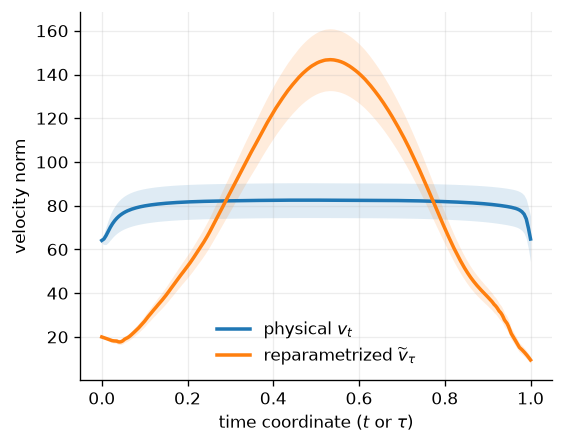}
        
    \end{subfigure}
    \caption 
    {CelebA64 flow matching model (teacher) measurements plotted against their respective coordinates.
    \textbf{Left}: estimated hardness before and after reparametrization.
    \textbf{Middle}: estimated map $\phi$.
    \textbf{Right}: original and reparameterized velocity norms.}
    \label{fig:celeb metrics}
\end{figure}

\paragraph{Normal-acceleration profiles.}

We evaluate normal-acceleration profiles along trajectories of the CIFAR-10 and CelebA64 teachers. Both exhibit a similar U-shaped temporal profile, with larger normal acceleration near the endpoints and smaller values in the middle. Figure~\ref{fig:celeb metrics} shows CelebA64 as a representative example; the CIFAR-10 profiles exhibit the same qualitative behavior, with different magnitudes.
This concentration motivates allocating more student
time to the boundary regions.
It is also qualitatively consistent with the U-shaped
distillation difficulty reported by
\citet{tian2026curvatureadaptive}, who learn an adaptive
time-sampling curriculum using reinforcement learning.
This agreement motivates our diagnostic, although it
does not establish a direct correspondence between
normal acceleration and distillation error.

Figure~\ref{fig:celeb metrics} illustrates the resulting
transformation for CelebA64.
The estimated clock makes the normal-acceleration
profile more uniform across time.
Meanwhile, the original teacher's velocity norm is
approximately constant over most of the interval,
whereas the warped velocity exhibits greater variation.
This observation motivates the physical-time
displacement scaling used for image distillation.

\paragraph{Teacher sampling.}
\begin{wrapfigure}[19]{r}{0.5\textwidth}
    \centering
    \includegraphics[width=0.5\textwidth]{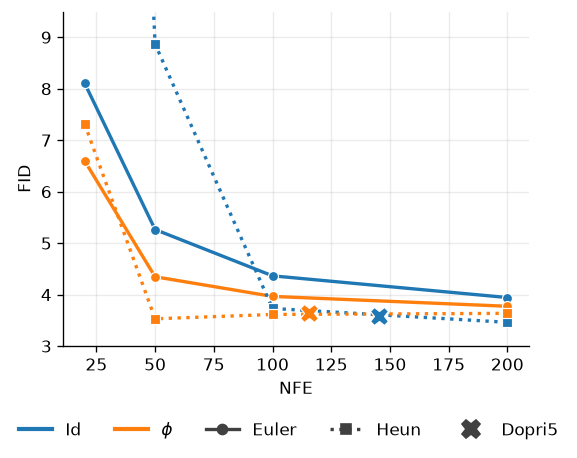}
    \caption{CIFAR-10 Flow Matching model (teacher) evaluated FID using different solvers and number of steps.}
    \label{fig:cifar10-teacher-fid-kid}
\end{wrapfigure}
Before evaluating distillation, we test whether the
estimated clock improves numerical sampling from the
frozen CIFAR-10 teacher.
For fixed-step solvers with $K$ integration steps,
we integrate the original field $b_t$ using either
uniform nodes $t_k=k/K$ or clock-adapted nodes
$t_k=\phi(k/K)$, for $k=0,\ldots,K$.
For the adaptive Dopri5 solver, we instead compare
integration of $b_t$ with integration of the
reparameterized field
$\widetilde b_\tau=\phi'(\tau)b_{\phi(\tau)}$.
We report FID against the number of velocity-network
evaluations (NFE), which dominate sampling cost.

Figure~\ref{fig:cifar10-teacher-fid-kid} shows that
clock-adapted grids improve sample quality at low
evaluation budgets.
With $25$ Heun steps, FID decreases from $8.8$ on
the uniform grid to $3.5$ on the clock-adapted grid,
approaching the teacher's performance at larger budgets.
Naturally, the advantage diminishes as the evaluation budget
increases and as we get closer to the real ODE's solutions.

These results support the usefulness of the estimated
clock for allocating a limited fixed-step integration
budget.
The following experiments test whether training in
the corresponding time coordinate also improves
flow-map distillation.

\subsection{Flow-Map Distillation}
\label{sec:distillation-experiments}

We compare distillation under the identity clock,
$\mathrm{Id}(t)=t$, with distillation under our estimated
clock $\phi$.
For each dataset, we use a fixed flow-matching teacher
and train two students with matched architectures,
training budgets, and optimization settings.
The $\mathrm{Id}$ student learns from $b$, while the
$\phi$ student learns from its reparameterized counterpart
$\widetilde b$.
Before distillation, we estimate $\phi$ from $256$
teacher trajectories generated with an adaptive ODE solver.
The clock remains fixed throughout training.
We use student-time displacement scaling for toy data
and physical-time displacement scaling for images,
as described in the method section.
All comparisons use matched numbers of student-network
evaluations.
Implementation details and additional results are
provided in the appendix.

\subsubsection{Toy Experiments}
\label{sec:toy-experiments}

\begin{figure}[H]
    \centering
    \includegraphics[width=1\linewidth]{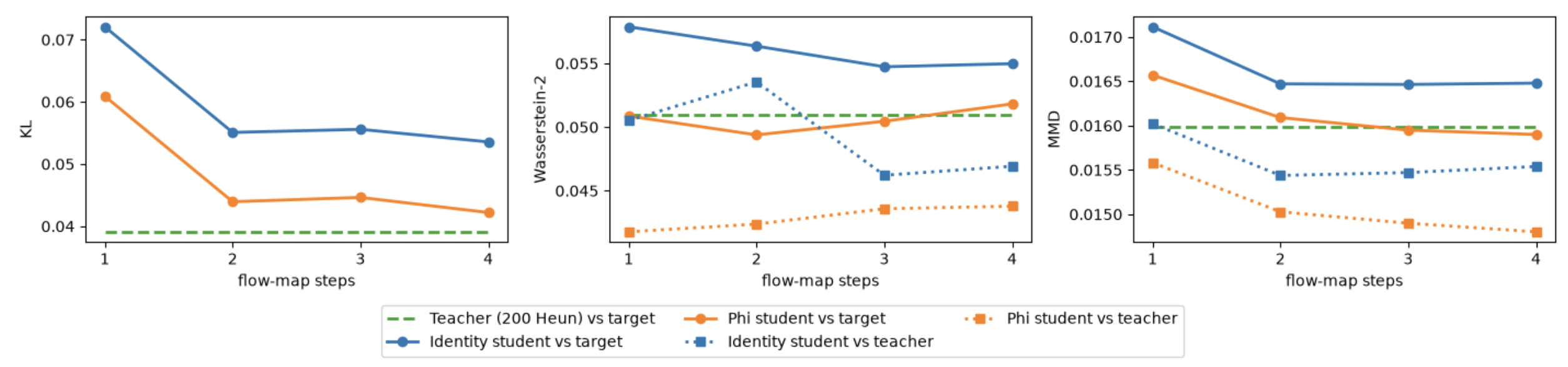}
    \caption{Quantitative results of the flow matching teacher and its distilled flow maps trained on the 2D checkerboard distribution.}
    \label{fig:toy metrics}
\end{figure}

We evaluate on a two-dimensional checkerboard distribution,
whose sharp boundaries and disconnected support provide
a test of sample fidelity.
The teacher and students use four-layer MLPs of width $512$,
trained for $50{,}000$ updates with a batch size of $8192$.
The source distribution is
$\rho_0=\mathcal N(\mathbf 0,\mI)$.

Figure~\ref{fig:toy samples} shows that $\phi$ produces
sharper boundaries and fewer samples outside the target
support than $\mathrm{Id}$.
We quantify sample quality using a histogram-based
forward KL statistic, Gaussian-kernel MMD with bandwidth
$\sigma=0.05$, and empirical Wasserstein-2 distance.
The KL computation exploits the known target density;
its discretization and normalization are specified in
Appendix~\ref{app:toy-metrics}.
MMD and Wasserstein-2 are evaluated against both the
target distribution and teacher-generated samples,
measuring target fidelity and distillation fidelity,
respectively.
Across the evaluated budgets of $1$, $2$, $3$, and $4$
network evaluations, $\phi$ improves all three reported
metrics relative to $\mathrm{Id}$
(Figure~\ref{fig:toy metrics}).

\subsubsection{Image Experiments}
\label{sec:image-experiments}

\begin{figure}
    \centering
    \begin{subfigure}{0.45\linewidth}
    \centering
        \includegraphics[width=\linewidth]{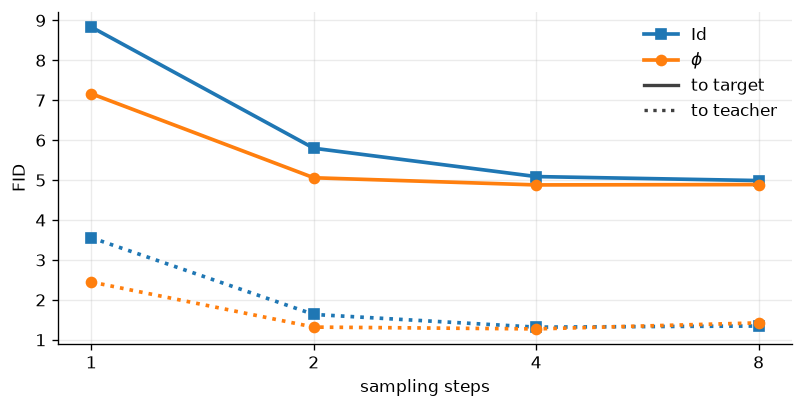}
        
    \end{subfigure}
    \hfill
    \begin{subfigure}{0.45\linewidth}
    \centering
        \includegraphics[width=\linewidth]{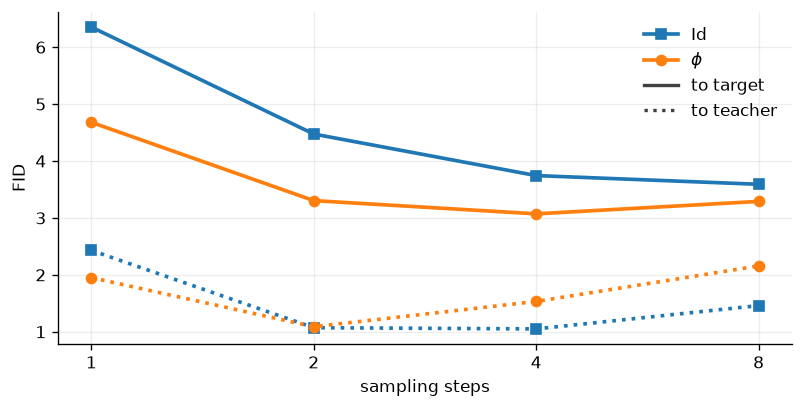}
        
    \end{subfigure}
    \caption{\textbf{Left}: CIFAR-10 Flow Map FID after 300k training steps.
    \textbf{Right}: CelebA64 Flow Map FID after 350k training steps.
    }
    \label{fig:image fid results}
\end{figure}

We evaluate unconditional generation on CIFAR-10
\citep{krizhevsky2009cifar} and CelebA64
\citep{liu2015celeba}.
Both experiments use RAdam \citep{liu2019radam}
and learned loss weighting following EDM2
\citep{karras2024analyzingimprovingtrainingdynamics}
and its adaptation to flow maps
\citep{boffi2025buildconsistencymodellearning}.
We report FID here and provide full FID/KID tables,
including additional checkpoints, in
Appendix~\ref{app:image-experiments}.

\begin{wrapfigure}[19]{r}{0.45\textwidth}
    \centering\includegraphics[width=0.45\textwidth]{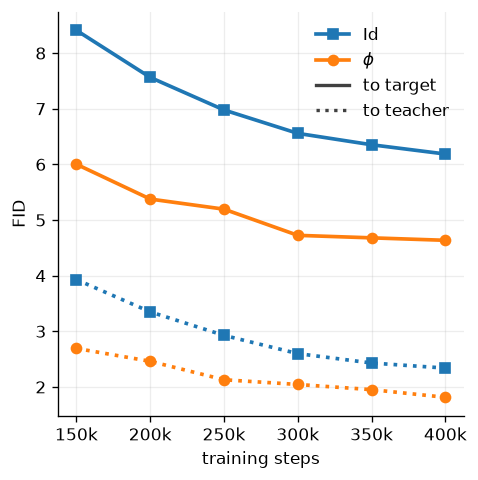}
    \caption{CelebA64 1 step FID vs training budget.}
    \label{fig: celeb FID vs training steps}
\end{wrapfigure}
\paragraph{CIFAR-10.}

We use approximately $35$M-parameter U-Nets and a
pretrained TorchCFM teacher
\citep{tong2023simulation,tong2024improving}.
Students are trained for $400{,}000$ updates with a
global batch size of $196$.

\paragraph{CelebA64.}
We use approximately $124$M-parameter EDM2 U-Nets.
The teacher is trained on the official CelebA training
partition for $350{,}000$ updates, achieving FID $2.36$
with $100$ Heun integration steps.
Students are trained for $400{,}000$ updates with a
global batch size of $256$.

\paragraph{Results.}
Figure~\ref{fig:image fid results} compares
$\mathrm{Id}$ and $\phi$ at matched inference budgets.
On both datasets, $\phi$ improves FID against the target
distribution at every displayed budget.
This does not imply uniformly better agreement with the
teacher distribution: at the displayed checkpoints,
$\mathrm{Id}$ achieves lower teacher-relative FID at eight
steps on CIFAR-10 and at two, four, and eight steps on
CelebA64.
The one-step comparison is particularly informative:
both clocks have identical endpoints and displacement
scales, so improvements reflect the learned map rather than the placement of intermediate sampling times. At $400$K training updates, one-step target FID decreases
from $8.57$ to $6.99$ on CIFAR-10 and from $6.19$ to $4.64$ on CelebA64 (Appendix~\ref{app:image-experiments}).Moreover Figure~\ref{fig: celeb FID vs training steps} shows that the 1 step improvement is consistent across training steps for the CelebA64 model (CIFAR-10 shows a very similar trend).

%% file: conclusion.tex
\section{Conclusion}

We proposed a geometry-aware time reparameterization that
redistributes normal acceleration across time to improve
flow-map distillation.
The clock is estimated once, preserves teacher trajectories,
and requires neither teacher retraining nor additional
inference-time network evaluations.
Experiments show improved few-step generation, in particular
one-step gains that highlight the role of time coordinates
in learning the map itself.
Establishing a theoretical link between normal acceleration
and distillation difficulty remains a direction
for future work.

\section*{AI Use Statement}
We used ChatGPT and AI coding assistants to support
implementation and debugging, literature searches,
manuscript drafting and revision, and figure preparation.
We also used these tools to discuss and refine mathematical
arguments and experimental methodology. The authors take responsibility for the final manuscript, code, references, and reported findings.

The authors reviewed and revised AI-assisted text, checked mathematical arguments and cited sources, and tested AI-assisted code against the stated methods.

%% file: appendix_proofs.tex
\section{Mathematical Details and Proofs}

\subsection{Proof of the main Theorem and construction of $\phi$}

\begin{proof}[Proof of Theorem~\ref{thm: phi construction}]
\label{prf:main thm}
Let $t\in[0,1]$ denote teacher time and $\tau\in[0,1]$
denote student time. Write
\[
    Z:=\int_0^1\sqrt{h_p(u)}\,\dd u.
\]
Since $h_p$ is positive and continuously differentiable,
$0<Z<\infty$ and
\[
    \psi'(t)=\frac{\sqrt{h_p(t)}}{Z}>0.
\]
Thus $\psi$ is a strictly increasing $C^2$ map with
$\psi(0)=0$ and $\psi(1)=1$.
Its inverse $\phi=\psi^{-1}$ is also $C^2$, with
$\phi'>0$, and satisfies
\[
    t=\phi(\tau),
    \qquad
    \tau=\psi(t).
\]

Consider a teacher trajectory $\rvx$ and its reparameterized
version
\[
    \rvy_\tau=\rvx_{\phi(\tau)}.
\]

By the chain rule, \[ \dot{\rvy}_\tau = \phi'(\tau)\dot{\rvx}_{\phi(\tau)}, \qquad \ddot{\rvy}_\tau = \phi''(\tau)\dot{\rvx}_{\phi(\tau)} + \phi'(\tau)^2\ddot{\rvx}_{\phi(\tau)}. \] Since $\phi'(\tau)>0$, the two trajectories have the same unit tangent at corresponding times. The term containing $\phi''$ is tangential, so \[ \widetilde a_\tau^\perp(\rvy_\tau) = \phi'(\tau)^2 a_{\phi(\tau)}^\perp(\rvx_{\phi(\tau)}), \qquad \widetilde n_\tau(\rvy_\tau) = \phi'(\tau)^2 n_{\phi(\tau)}(\rvx_{\phi(\tau)}). \] The same deterministic clock applies to every trajectory, hence \[ \widetilde h_p(\tau) = \phi'(\tau)^2 h_p(\phi(\tau)). \]

Here both expectations are taken over the same initial
distribution of teacher trajectories.

Finally, differentiating $\psi(\phi(\tau))=\tau$ gives
\[
    \psi'(\phi(\tau))\phi'(\tau)=1,
    \qquad
    \phi'(\tau)
    =
    \frac{Z}{\sqrt{h_p(\phi(\tau))}}.
\]
Substituting this expression into the preceding identity,
we obtain
\[
    \widetilde h_p(\tau)
    =
    \frac{Z^2}{h_p(\phi(\tau))}h_p(\phi(\tau))
    =
    \left(\int_0^1\sqrt{h_p(u)}\,\mathrm du\right)^2.
\]
which is independent of $\tau$.
\end{proof}

%% file: Additional_experimental_details.tex
\section{Toy Experiments: Implementation Details}
\label{app:toy-experiments}

\begin{figure}[!htb]
    \centering
    \includegraphics[width=1\linewidth]{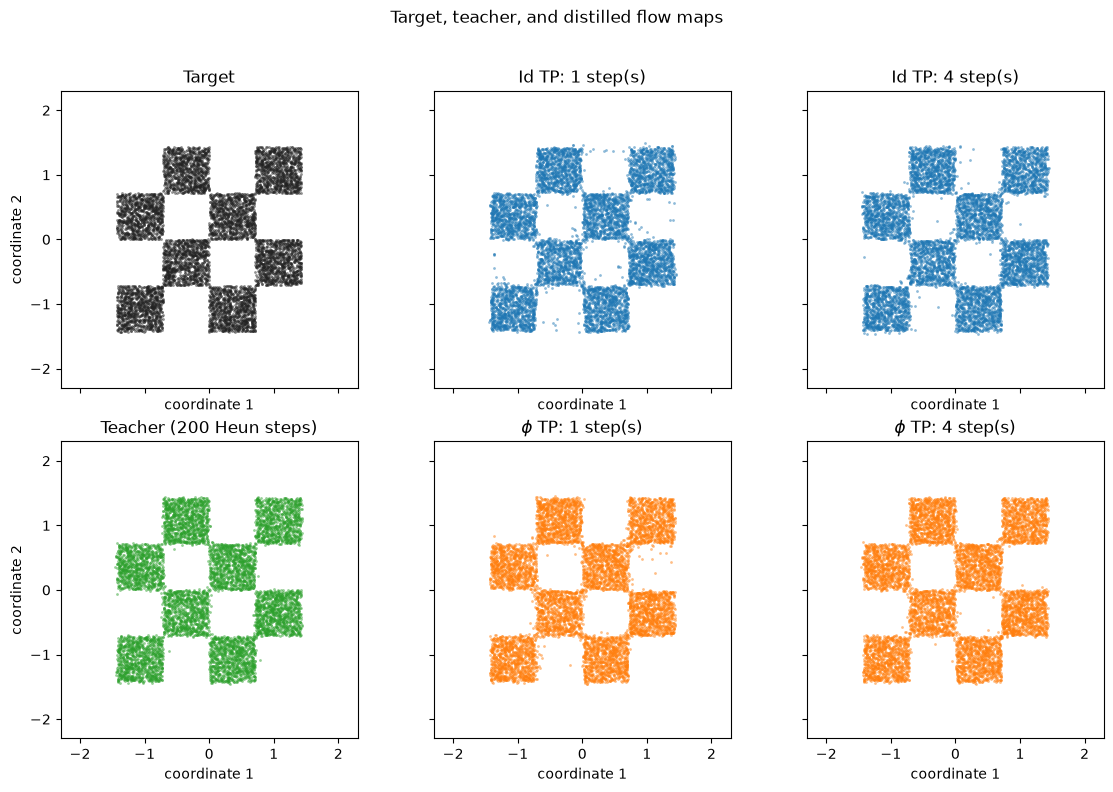}
    \caption{8192 samples from the target distribution (checkerboard, top left, black), the flow matching teacher (bottom left, green), the $\mathrm{Id}$-TP student (baseline, top middle and right) and the $\phi$-TP student (ours, bottom middle and right).}
    \label{fig:toy samples}
\end{figure}

\paragraph{Dataset and training.}
The target is a uniform distribution on the occupied
squares of a two-dimensional checkerboard.
The source distribution is
$\rho_0=\mathcal N(\mathbf 0,\mI)$.

The teacher and students use four-layer MLPs of width
$512$ and SiLU activation function.
Each model is trained for 50,000 optimizer updates with a batch size of 8192. The teacher is optimized using AdamW \citep{loshchilov2019decoupledweightdecayregularization} with learning rate $10^{-3}$, $\beta_1=0.9$, $\beta_2 = 0.999$, weight decay $10^{-4}$, and $\epsilon = 10^{-8}$. Its learning rate follows a cosine schedule down to $10^{-6}$, and its exponential moving average uses decay $0.998$. The students use AdamW with learning rate $10^{-2}$, the same Adam coefficients, weight decay, and numerical stabilization, followed by an exponential learning-rate schedule with multiplicative decay $0.9999$ after each optimizer update. Their exponential moving average
uses decay $0.999$. The maximum gradient norms are $10$ for the teacher and $1$ for the students. EMA
weights are used for validation and evaluation.
We use a constant loss weight $\frac{1}{2}$.
The $\phi$ student uses displacement scale
$\beta-\alpha$.

\paragraph{Evaluation protocol.}
\label{app:toy-metrics}
We evaluate generated samples using a histogram-based forward KL, Gaussian-kernel MMD, and empirical
Wasserstein-2 distance.
KL uses $64{,}000$ generated samples and the known target density. For MMD and Wasserstein-2, we use $8192$ samples per distribution and compare against both target samples and teacher-generated samples, measuring target fidelity and distillation fidelity, respectively.

\paragraph{Histogram-based forward KL.}
We estimate the forward divergence
$D_{\mathrm{KL}}(p\|q)$ using the known checkerboard
density $p$ and a histogram estimate of the generated
density $q$.
A fixed domain containing the target support is
partitioned into $64\times64$ equal-area bins.
For bin center $c_{ij}$, bin area $\Delta A$, and
generated count $n_{ij}$, we set
\[
    \widehat q(c_{ij})
    =
    \frac{n_{ij}}{N\Delta A},
\]
where $N$ is the total number of generated samples,
including those outside the histogram domain.
We compute
\[
    \widehat D_{\mathrm{KL}}^{\,\varepsilon}
    =
    \sum_{i,j:\,p(c_{ij})>0}
    p(c_{ij})\Delta A
    \log\!\left(
        \frac{p(c_{ij})}
        {\widehat q(c_{ij})+\varepsilon}
    \right),
    \qquad \varepsilon=10^{-10}.
\]
Normalization by $N$ retains the penalty associated
with generated mass outside the domain.
The density floor prevents infinite values from empty
target-positive bins.
This statistic approximates the forward KL at the
chosen spatial resolution; its value depends on
binning, sample count, and regularization.
We use identical evaluation settings for all models.

\paragraph{Maximum mean discrepancy.}
We use the Gaussian kernel
\[
    k(\vx,\vy)
    =
    \exp\!\left(
        -\frac{\norm{\vx-\vy}^2}{2\sigma^2}
    \right),
    \qquad \sigma=0.05.
\]
For distributions $\mu$ and $\nu$, the squared
population discrepancy is
\[
    \operatorname{MMD}_k^2(\mu,\nu)
    =
    \mathbb E[k(X,X')]
    +
    \mathbb E[k(Y,Y')]
    -
    2\mathbb E[k(X,Y)],
\]
where $X,X'\sim\mu$ and $Y,Y'\sim\nu$ are independent.
We report the biased estimate
\[
\widehat{\mathrm{MMD}}_{\mathrm b}=\left(
\frac{1}{n^2}\sum_{i,j}k(x_i,x_j)
+
\frac{1}{m^2}\sum_{i,j}k(y_i,y_j)
-
\frac{2}{nm}\sum_{i,j}k(x_i,y_j).\right)^{1/2}.
\]
\paragraph{Empirical Wasserstein distance.}
We compare empirical sample measures using the
quadratic transport cost:
\[
    W_2^2(\widehat\mu,\widehat\nu)
    =
    \min_{\gamma\in\Pi(\widehat\mu,\widehat\nu)}
    \int \norm{\vx-\vy}^2
    \,\mathrm d\gamma(\vx,\vy).
\]
For two equally sized point clouds with $n$ samples, we solve the assignment problem with cost matrix $C_{ij}=\norm{x_i-y_j}^2$ using SciPy's \texttt{linear\_sum\_assignment}. If $\sigma$ is the resulting optimal permutation, we report \[ W_2(\widehat\mu,\widehat\nu) = \left( \frac{1}{n}\sum_{i=1}^n \norm{x_i-y_{\sigma(i)}}^2 \right)^{1/2}. \]

\section{Image Experiments: Implementation and Additional Results}
\label{app:image-experiments}

\subsection{Shared distillation setup}
\label{app:image-shared}

\paragraph{Training inputs and time sampling.}
Both image experiments use unconditional students trained
with the stop-gradient Lagrangian distillation objective.
The teacher and, where applicable, the estimated clock
remain fixed throughout distillation.
Input states are sampled directly from the linear
interpolant: $I_\alpha$ for $\mathrm{Id}$ and
$I_{\phi(\alpha)}$ for $\phi$.

We draw $u,v\sim\mathcal U([0,1])$ independently and set
\[
    s=\min(u,v),\qquad t=\max(u,v).
\]
For $\mathrm{Id}$, the student receives $(s,t)$; for $\phi$,
it receives $(\alpha,\beta)=(\psi(s),\psi(t))$.
Both use the input state $I_s$.
The physical-time pair $(s,t)$ is uniform over the upper
triangle; the transformed pair $(\alpha,\beta)$ generally
is not.
We do not explicitly include diagonal samples.

In preliminary experiments, directly sampling and sorting
uniform student times also improved one- and two-step target
FID over $\mathrm{Id}$.
Sampling uniformly in physical time yielded further improvements
and is used for all image results reported here.

\paragraph{Learned loss weighting.}
We use the learned weighting introduced for diffusion
training in EDM2
\citep{karras2024analyzingimprovingtrainingdynamics}
and adapted to two-time flow-map objectives by
\citet{boffi2025buildconsistencymodellearning}.
Let $\ell_{\alpha,\beta}$ denote the unweighted
per-example squared distillation residual.
A scalar head with parameters $\eta$ predicts
\[
    w_{\alpha,\beta}
    :=
    w_\eta(\alpha,\beta-\alpha).
\]
We jointly optimize the student and weighting head using
\begin{equation}
\label{eq:image-learned-weighting}
    \mathcal L_{\mathrm{weighted}}(\vtheta,\eta)
    =
    \mathbb E
    \left[
        e^{-w_{\alpha,\beta}}\ell_{\alpha,\beta}
        +w_{\alpha,\beta}
    \right].
\end{equation}
The exponential factor rescales the contribution of each
time pair, while the additive term penalizes arbitrarily
suppressing the residual through large positive weights.
We found this weighting important for training stability.
Each student learns its own weighting function, using
the same weighting scheme for $\mathrm{Id}$ and $\phi$.
The weighting head is used only during training and adds
no inference-time network evaluations.

\paragraph{Hyperparameter selection.} 
For each image dataset, we compared candidate learning rates and decay schedules using FID and KID computed from 50,000
generated samples at multiple training checkpoints and
sampling budgets.
Both metrics consistently selected the same learning-rate
and decay configuration across these checkpoints and budgets.
The selected configuration was also the same for
$\mathrm{Id}$ and $\phi$.

\paragraph{Student evaluation.}
We evaluate EMA student weights, using decay $0.9999$
in both image experiments.
Comparisons between $\mathrm{Id}$ and $\phi$ use matched training checkpoints and numbers of student network evaluations.
The main text reports FID \citep{heusel2018ganstrainedtimescaleupdate}; the additional tables report both FID and KID \citep{bińkowski2021demystifyingmmdgans} at multiple training stages. Both statistics are computed using 50,000 generated samples and legacy TensorFlow-compatible Inception-v3 features. For the KID we use a standard cubic polynomial kernel 
\[
k(x,y)=\left(\frac{x^\top y}{d}+1\right)^3,
\] 
where $d$ is the Inception feature dimension. We report the mean unbiased estimate of $\operatorname{MMD}_k^2$ over $100$ subsets of $1000$ samples per distribution.

Each reported configuration uses a single training run.
Preliminary repetitions of selected configurations yielded very similar training curves and negligible differences between evaluation results.

\paragraph{Construction of $\phi$}
We used the parameters $p=2$, $\delta=1$, $\xi=10^{-5}$, $N=100$ as described in \ref{alg:clock-construction}. In practice, we found that the parameter $p$ does not have a major effect on the resulting $\phi$ (for tested values in $[1,4]$). Moreover the numerical floor $\xi$ has negligible effect in all our experiments since the estimate $h_p$ has very high values (in range $10^2$ and above).

\subsection{CIFAR-10}
\label{app:cifar10}

\input{Tables/cifar10-student-final-fid-kid_new}

\paragraph{Data.}
We use the $50{,}000$ training images of CIFAR-10
\citep{krizhevsky2009cifar} at their native
$32\times32$ resolution, without class conditioning.
The source distribution is
$\rho_0=\mathcal N(\mathbf 0,\mI)$.
We add horizontal flips. The images are then mapped to [-1,1].

\paragraph{Architecture.}
The teacher and students use the guided-diffusion U-Net architecture of \citet{dhariwal2021diffusion}, as implemented in TorchCFM \citep{tong2023simulation,tong2024improving}, with base width $128$ and channel multipliers $(1,2,2,2)$. As in \citet{boffi2025flowmap}, the student adapts the time-conditioning pathway to accept two times. 
The corresponding resolution--channel pairs are
\[
    (32^2,128),\quad
    (16^2,256),\quad
    (8^2,256),\quad
    (4^2,256).
\]
The backbone uses two encoder residual blocks and three
skip-connected decoder blocks per resolution, with
self-attention at $16\times16$ and in the bottleneck.
Attention uses four heads with $64$ channels per head.
The models contain approximately $35$M parameters.

The teacher embeds physical time $t$, whereas the student
embeds $\alpha$ and $\beta$ separately, using
$128$-dimensional Fourier features with time-input
scale $1000$.
These features are projected into the residual-block
conditioning stream.

\paragraph{Pretrained teacher.}
We use a publicly released TorchCFM checkpoint
\citep{tong2023simulation,tong2024improving}.
The supplied training configuration uses Adam with
learning rate $2\cdot10^{-4}$, default optimizer
parameters, a batch size of $128$, $5{,}000$ learning-rate
warmup updates, gradient-norm clipping at $1.0$,
EMA decay $0.9999$, and $400{,}000$ training updates.
The exact checkpoint is TorchCFM release v1.0.4,
\texttt{cfm\_cifar10\_weights\_step\_400000.pt}, trained using
Independent Conditional Flow Matching (I-CFM) with independent
source--target coupling and noise $\sigma=0$.
The teacher achieves approximately 50k-FID $3.5$ (see Fig~\ref{fig:cifar10-teacher-fid-kid})

\paragraph{Student optimization.}
Both students use RAdam \citep{liu2019radam},
a global batch size of $196$, and $400{,}000$ optimizer
updates.
The learning rate at optimizer update $k$ is \[ \eta(k)=2\cdot10^{-4}\,0.992^{k/1000}. \]
Both students use the learned weighting from
Appendix~\ref{app:image-shared}.

\paragraph{Evaluation.}
The main-text figure uses the checkpoint at
300k training steps.
Full results across inference budgets and training
steps are reported in Table~\ref{tab: cifar10 full eval}.

\subsection{CelebA64}
\label{app:celeba64}

\input{Tables/celeba64_metric_results_new}

\paragraph{Data and preprocessing.}
We use the official CelebA training partition
\citep{liu2015celeba}, containing $162{,}770$
images.
Validation uses $3{,}072$ fixed images sampled with
seed $0$ from the official validation partition.
The official test partition is not used for training.
Images are center-cropped or padded to $140\times140$,
resized to $64\times64$ using area interpolation, and
converted from uint8 RGB to $[-1,1]$.
We use no horizontal flipping or other online augmentation.

The source distribution is
$\rho_0=\mathcal N(\mathbf 0,0.25\mI)$.
Source and data samples are coupled independently.
Teacher training uses uniformly sampled physical times
and linear interpolation without additional noise.

\paragraph{Architecture.}
We use the NVLabs EDM2 U-Net backbone
\citep{karras2024analyzingimprovingtrainingdynamics}
with base width $128$ and channel multipliers
$(1,2,3,4)$, giving
\[
    (64^2,128),\quad
    (32^2,256),\quad
    (16^2,384),\quad
    (8^2,512).
\]
There are three encoder residual blocks and four
skip-connected decoder blocks per resolution, with
attention at $16\times16$ and $8\times8$.
Teacher dropout is zero; student dropout is $0.1$.
The teacher and student U-Nets contain $124{,}198{,}821$
and $124{,}264{,}357$ trainable parameters, respectively.

The backbone uses magnitude-preserving convolutions,
pixel normalization, magnitude-preserving residual and
skip mixing, normalized attention, and learnable
residual and embedding gains.
An internal constant-one plane is appended to RGB,
so the first convolution receives four channels.
No spatial time-conditioning planes are used.

The teacher uses a deterministic $128$-dimensional
Fourier embedding of physical time,
\[
    z(t)=\sqrt{2}\,[\cos(t\omega),\sin(t\omega)],
\]
with logarithmically spaced frequencies and maximum
period $10^4$.
A magnitude-preserving projection maps this embedding
to a $512$-dimensional conditioning stream.
Students instead concatenate separate $128$-dimensional
embeddings of $\alpha$ and $\beta$ and project the result
to $512$ dimensions.
After magnitude-preserving SiLU, the student embedding
modulates every residual block through the original
conditioning pathway.

\paragraph{Teacher optimization.}
The teacher minimizes the independent conditional
flow-matching loss, averaged over batch, channels,
and spatial dimensions, without learned loss weighting.
We use RAdam with
$(\beta_1,\beta_2)=(0.9,0.999)$ and $\epsilon=10^{-8}$,
a global batch size of $256$, and trained for $350{,}000$ updates.
The learning-rate schedule is
\[
    \eta(k)
    =
    \frac{10^{-2}}{\sqrt{\max(k/35000,1)}}.
\]
Training uses four A100 GPUs with $64$ images per GPU,
BF16 computation, gradient-norm clipping at $1.0$,
and EMA decay $0.9999$.
EDM2 weights are renormalized after every optimizer update.

We freeze the step-$350{,}000$ EMA state as an immutable,
hashed snapshot for distillation.
Its evaluation with $100$ Heun steps and $50{,}000$
generated samples gives approximately FID $2.36$
and KID $1.46\cdot10^{-3}$.

\paragraph{Student optimization and weighting implementation.}
Both students use one RAdam optimizer with separate
parameter groups for the U-Net and weighting head.
The groups share the same learning rate,
$(\beta_1,\beta_2)=(0.9,0.999)$,
$\epsilon=10^{-8}$, and zero weight decay.
Students are trained for $400{,}000$ updates with learning-rate schedule \[ \eta(k)=3\cdot10^{-3}\,0.992^{k/1000}. \]

The global batch size is $256$, with $64$ examples
per GPU on four GPUs.
We use BF16 computation, gradient-norm clipping at $1.0$,
and EMA decay $0.9999$.
EDM2 weights are renormalized after every optimizer
update and every EMA update.
There is no early stopping.

For Equation~\ref{eq:image-learned-weighting},
$\ell_{\alpha,\beta}$ is a sum over all
$3\cdot64\cdot64$ image coordinates.
Weighted losses are then averaged over the batch.
The weighting head embeds $\alpha$ and
$\beta-\alpha$ separately with $128$-dimensional
sinusoidal features, concatenates them, and applies
a zero-initialized linear projection to one scalar.
It contains $257$ parameters in addition to the U-Net.

%% file: Tables/cifar10-student-final-fid-kid_new.tex
\definecolor{bestcell}{HTML}{D9EAD3}
\begin{table}
\begin{tabular}{llrrrrrrrrrr}
\toprule
Training & Model & \multicolumn{5}{c}{FID to target} & \multicolumn{5}{c}{FID to teacher} \\
\cmidrule(lr){3-7} \cmidrule(lr){8-12}
 & & 1 & 2 & 4 & 8 & 20 & 1 & 2 & 4 & 8 & 20 \\
\midrule
150k & Id & 10.09 & 6.31 & 5.54 & 5.49 & 5.65 & 4.44 & 1.83 & \cellcolor{bestcell}\textbf{1.44} & \cellcolor{bestcell}\textbf{1.52} & 1.75 \\
150k & $\phi$ (teacher time) & \cellcolor{bestcell}\textbf{7.94} & \cellcolor{bestcell}\textbf{5.60} & \cellcolor{bestcell}\textbf{5.34} & \cellcolor{bestcell}\textbf{5.45} & \cellcolor{bestcell}\textbf{5.17} & \cellcolor{bestcell}\textbf{2.85} & \cellcolor{bestcell}\textbf{1.49} & 1.44 & 1.69 & \cellcolor{bestcell}\textbf{1.66} \\
\addlinespace
300k & Id & 8.83 & 5.80 & 5.09 & 4.99 & 5.08 & 3.56 & 1.64 & 1.32 & \cellcolor{bestcell}\textbf{1.35} & \cellcolor{bestcell}\textbf{1.48} \\
300k & $\phi$ (teacher time) & \cellcolor{bestcell}\textbf{7.16} & \cellcolor{bestcell}\textbf{5.06} & \cellcolor{bestcell}\textbf{4.88} & \cellcolor{bestcell}\textbf{4.89} & \cellcolor{bestcell}\textbf{4.85} & \cellcolor{bestcell}\textbf{2.45} & \cellcolor{bestcell}\textbf{1.32} & \cellcolor{bestcell}\textbf{1.27} & 1.43 & 1.48 \\
\addlinespace
400k & Id & 8.57 & 5.71 & 4.98 & 4.89 & 5.02 & 3.40 & 1.62 & 1.30 & \cellcolor{bestcell}\textbf{1.31} & 1.45 \\
400k & $\phi$ (teacher time) & \cellcolor{bestcell}\textbf{6.99} & \cellcolor{bestcell}\textbf{5.01} & \cellcolor{bestcell}\textbf{4.80} & \cellcolor{bestcell}\textbf{4.80} & \cellcolor{bestcell}\textbf{4.79} & \cellcolor{bestcell}\textbf{2.38} & \cellcolor{bestcell}\textbf{1.32} & \cellcolor{bestcell}\textbf{1.25} & 1.37 & \cellcolor{bestcell}\textbf{1.45} \\
\bottomrule
\end{tabular}
\begin{tabular}{llrrrrrrrrrr}
\toprule
Training & Model & \multicolumn{5}{c}{KID to target ($\times 10^3$)} & \multicolumn{5}{c}{KID to teacher ($\times 10^3$)} \\
\cmidrule(lr){3-7} \cmidrule(lr){8-12}
 & & 1 & 2 & 4 & 8 & 20 & 1 & 2 & 4 & 8 & 20 \\
\midrule
150k & Id & 6.07 & 3.37 & 2.91 & \cellcolor{bestcell}\textbf{3.04} & 3.24 & 2.43 & 0.61 & \cellcolor{bestcell}\textbf{0.33} & \cellcolor{bestcell}\textbf{0.45} & 0.68 \\
150k & $\phi$ (teacher time) & \cellcolor{bestcell}\textbf{4.38} & \cellcolor{bestcell}\textbf{2.63} & \cellcolor{bestcell}\textbf{2.79} & 3.08 & \cellcolor{bestcell}\textbf{2.73} & \cellcolor{bestcell}\textbf{1.20} & \cellcolor{bestcell}\textbf{0.26} & 0.37 & 0.58 & \cellcolor{bestcell}\textbf{0.45} \\
\addlinespace
300k & Id & 5.43 & 3.16 & 2.71 & 2.75 & 2.90 & 1.91 & 0.52 & 0.28 & \cellcolor{bestcell}\textbf{0.31} & 0.48 \\
300k & $\phi$ (teacher time) & \cellcolor{bestcell}\textbf{3.90} & \cellcolor{bestcell}\textbf{2.37} & \cellcolor{bestcell}\textbf{2.53} & \cellcolor{bestcell}\textbf{2.72} & \cellcolor{bestcell}\textbf{2.73} & \cellcolor{bestcell}\textbf{0.97} & \cellcolor{bestcell}\textbf{0.19} & \cellcolor{bestcell}\textbf{0.23} & 0.40 & \cellcolor{bestcell}\textbf{0.45} \\
\addlinespace
400k & Id & 5.21 & 3.13 & 2.71 & \cellcolor{bestcell}\textbf{2.62} & 2.92 & 1.78 & 0.52 & 0.27 & \cellcolor{bestcell}\textbf{0.28} & 0.45 \\
400k & $\phi$ (teacher time) & \cellcolor{bestcell}\textbf{3.78} & \cellcolor{bestcell}\textbf{2.34} & \cellcolor{bestcell}\textbf{2.45} & 2.66 & \cellcolor{bestcell}\textbf{2.71} & \cellcolor{bestcell}\textbf{0.92} & \cellcolor{bestcell}\textbf{0.20} & \cellcolor{bestcell}\textbf{0.21} & 0.35 & \cellcolor{bestcell}\textbf{0.42} \\
\bottomrule
\end{tabular}

\caption{CIFAR-10 distillation results across training checkpoints and sampling budgets. KID values are multiplied by $10^3$. For each checkpoint, reference distribution, and sampling budget, the lower value across the two clocks is highlighted.}
\label{tab: cifar10 full eval}
\end{table}

%% file: Tables/celeba64_metric_results_new.tex
\definecolor{bestcell}{HTML}{D9EAD3}
\begin{table}[!htb]
\begin{tabular}{llrrrrrrrrrr}
\toprule
Training & Model & \multicolumn{5}{c}{FID to target} & \multicolumn{5}{c}{FID to teacher} \\
\cmidrule(lr){3-7} \cmidrule(lr){8-12}
 & & 1 & 2 & 4 & 8 & 20 & 1 & 2 & 4 & 8 & 20 \\
\midrule
150k & Id & 8.41 & 5.73 & 4.70 & 4.35 & 4.15 & 3.93 & 1.73 & \cellcolor{bestcell}\textbf{1.48} & \cellcolor{bestcell}\textbf{1.76} & \cellcolor{bestcell}\textbf{1.98} \\
150k & $\phi$ (teacher time) & \cellcolor{bestcell}\textbf{6.00} & \cellcolor{bestcell}\textbf{4.18} & \cellcolor{bestcell}\textbf{3.84} & \cellcolor{bestcell}\textbf{4.05} & \cellcolor{bestcell}\textbf{4.03} & \cellcolor{bestcell}\textbf{2.69} & \cellcolor{bestcell}\textbf{1.46} & 1.97 & 2.60 & 2.68 \\
\addlinespace
200k & Id & 7.56 & 5.14 & 4.18 & 3.92 & 3.82 & 3.35 & 1.44 & \cellcolor{bestcell}\textbf{1.25} & \cellcolor{bestcell}\textbf{1.57} & \cellcolor{bestcell}\textbf{1.81} \\
200k & $\phi$ (teacher time) & \cellcolor{bestcell}\textbf{5.38} & \cellcolor{bestcell}\textbf{3.72} & \cellcolor{bestcell}\textbf{3.44} & \cellcolor{bestcell}\textbf{3.63} & \cellcolor{bestcell}\textbf{3.62} & \cellcolor{bestcell}\textbf{2.46} & \cellcolor{bestcell}\textbf{1.39} & 1.89 & 2.45 & 2.52 \\
\addlinespace
250k & Id & 6.98 & 4.78 & 3.93 & 3.74 & 3.66 & 2.93 & 1.26 & \cellcolor{bestcell}\textbf{1.14} & \cellcolor{bestcell}\textbf{1.51} & \cellcolor{bestcell}\textbf{1.75} \\
250k & $\phi$ (teacher time) & \cellcolor{bestcell}\textbf{5.20} & \cellcolor{bestcell}\textbf{3.67} & \cellcolor{bestcell}\textbf{3.37} & \cellcolor{bestcell}\textbf{3.56} & \cellcolor{bestcell}\textbf{3.61} & \cellcolor{bestcell}\textbf{2.13} & \cellcolor{bestcell}\textbf{1.15} & 1.61 & 2.24 & 2.42 \\
\addlinespace
300k & Id & 6.56 & 4.54 & 3.75 & 3.60 & 3.58 & 2.60 & \cellcolor{bestcell}\textbf{1.11} & \cellcolor{bestcell}\textbf{1.06} & \cellcolor{bestcell}\textbf{1.46} & \cellcolor{bestcell}\textbf{1.74} \\
300k & $\phi$ (teacher time) & \cellcolor{bestcell}\textbf{4.72} & \cellcolor{bestcell}\textbf{3.29} & \cellcolor{bestcell}\textbf{3.08} & \cellcolor{bestcell}\textbf{3.33} & \cellcolor{bestcell}\textbf{3.38} & \cellcolor{bestcell}\textbf{2.04} & 1.15 & 1.65 & 2.29 & 2.43 \\
\addlinespace
350k & Id & 6.35 & 4.47 & 3.74 & 3.59 & 3.56 & 2.43 & \cellcolor{bestcell}\textbf{1.07} & \cellcolor{bestcell}\textbf{1.05} & \cellcolor{bestcell}\textbf{1.46} & \cellcolor{bestcell}\textbf{1.75} \\
350k & $\phi$ (teacher time) & \cellcolor{bestcell}\textbf{4.68} & \cellcolor{bestcell}\textbf{3.30} & \cellcolor{bestcell}\textbf{3.07} & \cellcolor{bestcell}\textbf{3.29} & \cellcolor{bestcell}\textbf{3.33} & \cellcolor{bestcell}\textbf{1.95} & 1.09 & 1.53 & 2.16 & 2.30 \\
\addlinespace
400k & Id & 6.19 & 4.35 & 3.63 & 3.51 & 3.51 & 2.34 & \cellcolor{bestcell}\textbf{1.02} & \cellcolor{bestcell}\textbf{1.01} & \cellcolor{bestcell}\textbf{1.44} & \cellcolor{bestcell}\textbf{1.74} \\
400k & $\phi$ (teacher time) & \cellcolor{bestcell}\textbf{4.64} & \cellcolor{bestcell}\textbf{3.36} & \cellcolor{bestcell}\textbf{3.12} & \cellcolor{bestcell}\textbf{3.28} & \cellcolor{bestcell}\textbf{3.26} & \cellcolor{bestcell}\textbf{1.81} & 1.05 & 1.49 & 2.11 & 2.26 \\
\bottomrule
\end{tabular}
\begin{tabular}{llrrrrrrrrrr}
\toprule
Training & Model & \multicolumn{5}{c}{KID to target ($\times 10^3$)} & \multicolumn{5}{c}{KID to teacher ($\times 10^3$)} \\
\cmidrule(lr){3-7} \cmidrule(lr){8-12}
 & & 1 & 2 & 4 & 8 & 20 & 1 & 2 & 4 & 8 & 20 \\
\midrule
150k & Id & 5.43 & 3.55 & 3.06 & \cellcolor{bestcell}\textbf{3.01} & \cellcolor{bestcell}\textbf{2.91} & 2.72 & \cellcolor{bestcell}\textbf{1.05} & \cellcolor{bestcell}\textbf{1.17} & \cellcolor{bestcell}\textbf{1.69} & \cellcolor{bestcell}\textbf{1.92} \\
150k & $\phi$ (teacher time) & \cellcolor{bestcell}\textbf{3.79} & \cellcolor{bestcell}\textbf{2.65} & \cellcolor{bestcell}\textbf{2.76} & 3.08 & 3.06 & \cellcolor{bestcell}\textbf{1.97} & 1.14 & 2.03 & 2.73 & 2.71 \\
\addlinespace
200k & Id & 4.84 & 3.13 & 2.68 & \cellcolor{bestcell}\textbf{2.71} & \cellcolor{bestcell}\textbf{2.72} & 2.32 & \cellcolor{bestcell}\textbf{0.86} & \cellcolor{bestcell}\textbf{1.01} & \cellcolor{bestcell}\textbf{1.54} & \cellcolor{bestcell}\textbf{1.79} \\
200k & $\phi$ (teacher time) & \cellcolor{bestcell}\textbf{3.43} & \cellcolor{bestcell}\textbf{2.38} & \cellcolor{bestcell}\textbf{2.49} & 2.77 & 2.74 & \cellcolor{bestcell}\textbf{1.91} & 1.17 & 1.99 & 2.61 & 2.57 \\
\addlinespace
250k & Id & 4.43 & 2.93 & 2.53 & \cellcolor{bestcell}\textbf{2.60} & \cellcolor{bestcell}\textbf{2.60} & 1.99 & \cellcolor{bestcell}\textbf{0.75} & \cellcolor{bestcell}\textbf{0.93} & \cellcolor{bestcell}\textbf{1.48} & \cellcolor{bestcell}\textbf{1.73} \\
250k & $\phi$ (teacher time) & \cellcolor{bestcell}\textbf{3.17} & \cellcolor{bestcell}\textbf{2.21} & \cellcolor{bestcell}\textbf{2.31} & 2.67 & 2.73 & \cellcolor{bestcell}\textbf{1.50} & 0.83 & 1.60 & 2.34 & 2.44 \\
\addlinespace
300k & Id & 4.15 & 2.78 & 2.43 & 2.53 & 2.59 & 1.74 & \cellcolor{bestcell}\textbf{0.64} & \cellcolor{bestcell}\textbf{0.86} & \cellcolor{bestcell}\textbf{1.44} & \cellcolor{bestcell}\textbf{1.74} \\
300k & $\phi$ (teacher time) & \cellcolor{bestcell}\textbf{2.92} & \cellcolor{bestcell}\textbf{2.03} & \cellcolor{bestcell}\textbf{2.16} & \cellcolor{bestcell}\textbf{2.52} & \cellcolor{bestcell}\textbf{2.57} & \cellcolor{bestcell}\textbf{1.50} & 0.88 & 1.63 & 2.37 & 2.44 \\
\addlinespace
350k & Id & 4.01 & 2.76 & 2.43 & 2.54 & 2.59 & 1.61 & \cellcolor{bestcell}\textbf{0.62} & \cellcolor{bestcell}\textbf{0.86} & \cellcolor{bestcell}\textbf{1.44} & \cellcolor{bestcell}\textbf{1.75} \\
350k & $\phi$ (teacher time) & \cellcolor{bestcell}\textbf{2.84} & \cellcolor{bestcell}\textbf{2.01} & \cellcolor{bestcell}\textbf{2.10} & \cellcolor{bestcell}\textbf{2.46} & \cellcolor{bestcell}\textbf{2.51} & \cellcolor{bestcell}\textbf{1.37} & 0.79 & 1.50 & 2.22 & 2.29 \\
\addlinespace
400k & Id & 3.89 & 2.66 & 2.33 & 2.45 & 2.54 & 1.56 & \cellcolor{bestcell}\textbf{0.58} & \cellcolor{bestcell}\textbf{0.82} & \cellcolor{bestcell}\textbf{1.42} & \cellcolor{bestcell}\textbf{1.74} \\
400k & $\phi$ (teacher time) & \cellcolor{bestcell}\textbf{2.80} & \cellcolor{bestcell}\textbf{2.04} & \cellcolor{bestcell}\textbf{2.10} & \cellcolor{bestcell}\textbf{2.41} & \cellcolor{bestcell}\textbf{2.43} & \cellcolor{bestcell}\textbf{1.31} & 0.80 & 1.49 & 2.18 & 2.26 \\
\bottomrule
\end{tabular}
\caption{CelebA64 distillation results across training checkpoints and sampling budgets. KID values are multiplied by $10^3$. For each checkpoint, reference distribution, and sampling budget, the lower value across the two clocks is highlighted.}
\label{tab: celeba full eval}
\end{table}

%% file: iclr2027_conference.bib
@inproceedings{sohldickstein2015deep,
  title     = {Deep Unsupervised Learning using Nonequilibrium Thermodynamics},
  author    = {Sohl-Dickstein, Jascha and Weiss, Eric and
               Maheswaranathan, Niru and Ganguli, Surya},
  booktitle = {Proceedings of the 32nd International Conference on Machine Learning},
  series    = {Proceedings of Machine Learning Research},
  volume    = {37},
  pages     = {2256--2265},
  year      = {2015},
  publisher = {PMLR},
  url       = {https://proceedings.mlr.press/v37/sohl-dickstein15.html}
}

@inproceedings{ho2020denoising,
  title     = {Denoising Diffusion Probabilistic Models},
  author    = {Ho, Jonathan and Jain, Ajay and Abbeel, Pieter},
  booktitle = {Advances in Neural Information Processing Systems},
  volume    = {33},
  pages     = {6840--6851},
  year      = {2020},
  url       = {https://proceedings.neurips.cc/paper/2020/hash/4c5bcfec8584af0d967f1ab10179ca4b-Abstract.html}
}

@inproceedings{song2021score,
  title     = {Score-Based Generative Modeling through Stochastic Differential Equations},
  author    = {Song, Yang and Sohl-Dickstein, Jascha and Kingma, Diederik P.
               and Kumar, Abhishek and Ermon, Stefano and Poole, Ben},
  booktitle = {International Conference on Learning Representations},
  year      = {2021},
  url       = {https://openreview.net/forum?id=PxTIG12RRHS}
}

@inproceedings{dhariwal2021diffusion,
  title     = {Diffusion Models Beat {GANs} on Image Synthesis},
  author    = {Dhariwal, Prafulla and Nichol, Alexander},
  booktitle = {Advances in Neural Information Processing Systems},
  volume    = {34},
  pages     = {8780--8794},
  year      = {2021},
  url       = {https://proceedings.neurips.cc/paper/2021/hash/49ad23d1ec9fa4bd8d77d02681df5cfa-Abstract.html}
}

@inproceedings{rombach2022highresolution,
  title     = {High-Resolution Image Synthesis with Latent Diffusion Models},
  author    = {Rombach, Robin and Blattmann, Andreas and Lorenz, Dominik
               and Esser, Patrick and Ommer, Bj{\"o}rn},
  booktitle = {Proceedings of the IEEE/CVF Conference on Computer Vision
               and Pattern Recognition},
  pages     = {10684--10695},
  year      = {2022},
  url       = {https://openaccess.thecvf.com/content/CVPR2022/html/Rombach_High-Resolution_Image_Synthesis_With_Latent_Diffusion_Models_CVPR_2022_paper.html}
}

@inproceedings{song2021denoising,
  title     = {Denoising Diffusion Implicit Models},
  author    = {Song, Jiaming and Meng, Chenlin and Ermon, Stefano},
  booktitle = {International Conference on Learning Representations},
  year      = {2021},
  url       = {https://openreview.net/forum?id=St1giarCHLP}
}

@inproceedings{lu2022dpmsolver,
  title     = {{DPM-Solver}: A Fast {ODE} Solver for Diffusion Probabilistic
               Model Sampling in Around 10 Steps},
  author    = {Lu, Cheng and Zhou, Yuhao and Bao, Fan and Chen, Jianfei
               and Li, Chongxuan and Zhu, Jun},
  booktitle = {Advances in Neural Information Processing Systems},
  volume    = {35},
  pages     = {5775--5787},
  year      = {2022},
  url       = {https://proceedings.neurips.cc/paper_files/paper/2022/hash/260a14acce2a89dad36adc8eefe7c59e-Abstract-Conference.html}
}

@inproceedings{salimans2022progressive,
  title     = {Progressive Distillation for Fast Sampling of Diffusion Models},
  author    = {Salimans, Tim and Ho, Jonathan},
  booktitle = {International Conference on Learning Representations},
  year      = {2022},
  url       = {https://openreview.net/forum?id=TIdIXIpzhoI}
}

@inproceedings{song2023consistency,
  title     = {Consistency Models},
  author    = {Song, Yang and Dhariwal, Prafulla and Chen, Mark
               and Sutskever, Ilya},
  booktitle = {Proceedings of the 40th International Conference on Machine Learning},
  series    = {Proceedings of Machine Learning Research},
  volume    = {202},
  pages     = {32211--32252},
  year      = {2023},
  publisher = {PMLR},
  url       = {https://proceedings.mlr.press/v202/song23a.html}
}

@inproceedings{lipman2023flowmatching,
  title     = {Flow Matching for Generative Modeling},
  author    = {Lipman, Yaron and Chen, Ricky T. Q. and Ben-Hamu, Heli
               and Nickel, Maximilian and Le, Matthew},
  booktitle = {International Conference on Learning Representations},
  year      = {2023},
  url       = {https://openreview.net/forum?id=PqvMRDCJT9t}
}

@inproceedings{liu2023rectifiedflow,
  title     = {Flow Straight and Fast: Learning to Generate and Transfer
               Data with Rectified Flow},
  author    = {Liu, Xingchao and Gong, Chengyue and Liu, Qiang},
  booktitle = {International Conference on Learning Representations},
  year      = {2023},
  url       = {https://openreview.net/forum?id=XVjTT1nw5z}
}

@inproceedings{albergo2023stochasticinterpolants,
  title     = {Building Normalizing Flows with Stochastic Interpolants},
  author    = {Albergo, Michael S. and Vanden-Eijnden, Eric},
  booktitle = {International Conference on Learning Representations},
  year      = {2023},
  url       = {https://openreview.net/forum?id=li7qeBbCR1t}
}

@inproceedings{kim2024consistencytrajectory,
  title     = {Consistency Trajectory Models: Learning Probability
               Flow {ODE} Trajectory of Diffusion},
  author    = {Kim, Dongjun and Lai, Chieh-Hsin and Liao, Wei-Hsiang
               and Murata, Naoki and Takida, Yuhta and Uesaka, Toshimitsu
               and He, Yutong and Mitsufuji, Yuki and Ermon, Stefano},
  booktitle = {International Conference on Learning Representations},
  year      = {2024},
  url       = {https://openreview.net/forum?id=ymjI8feDTD}
}

@article{boffi2025flowmap,
  title   = {Flow Map Matching with Stochastic Interpolants:
             A Mathematical Framework for Consistency Models},
  author  = {Boffi, Nicholas M. and Albergo, Michael S.
             and Vanden-Eijnden, Eric},
  journal = {Transactions on Machine Learning Research},
  year    = {2025},
  url     = {https://openreview.net/forum?id=cqDH0e6ak2}
}

@inproceedings{frans2025shortcut,
  title     = {One Step Diffusion via Shortcut Models},
  author    = {Frans, Kevin and Hafner, Danijar and Levine, Sergey
               and Abbeel, Pieter},
  booktitle = {International Conference on Learning Representations},
  year      = {2025},
  url       = {https://openreview.net/forum?id=OlzB6LnXcS}
}

@inproceedings{geng2025meanflow,
  title     = {Mean Flows for One-step Generative Modeling},
  author    = {Geng, Zhengyang and Deng, Mingyang and Bai, Xingjian
               and Kolter, J. Zico and He, Kaiming},
  booktitle = {Advances in Neural Information Processing Systems},
  volume    = {38},
  year      = {2025},
  url       = {https://openreview.net/forum?id=uWj4s7rMnR}
}

@article{tong2024improving,
title={Improving and generalizing flow-based generative models with minibatch optimal transport},
author={Alexander Tong and Kilian FATRAS and Nikolay Malkin and Guillaume Huguet and Yanlei Zhang and Jarrid Rector-Brooks and Guy Wolf and Yoshua Bengio},
journal={Transactions on Machine Learning Research},
issn={2835-8856},
year={2024},
url={https://openreview.net/forum?id=CD9Snc73AW},
note={Expert Certification}
}

@article{tong2023simulation,
   title={Simulation-Free Schr{\"o}dinger Bridges via Score and Flow Matching},
   author={Tong, Alexander and Malkin, Nikolay and Fatras, Kilian and Atanackovic, Lazar and Zhang, Yanlei and Huguet, Guillaume and Wolf, Guy and Bengio, Yoshua},
   year={2023},
   journal={arXiv preprint 2307.03672}
}

@misc{tian2026curvatureadaptive,
      title={Curvature-Adaptive Consistency Flow Matching: Autonomous Trajectory Optimization via Reinforcement Learning}, 
      author={Songtao Tian and Guhan Chen and Bohan Li and Jingyi Ma and Zixiong Yu},
      year={2026},
      eprint={2606.22394},
      archivePrefix={arXiv},
      primaryClass={cs.CV},
      url={https://arxiv.org}, 
}

@inproceedings{liu2019radam,
  title={On the Variance of the Adaptive Learning Rate and Beyond},
  author={Liu, Liyuan and Jiang, Haoming and He, Pengcheng and Chen, Weizhu and Liu, Xiaodong and Gao, Jianfeng and Han, Jiawei},
  booktitle={Proceedings of the Eighth International Conference on Learning Representations (ICLR 2020)},
  year={2020}
}

@misc{karras2024analyzingimprovingtrainingdynamics,
      title={Analyzing and Improving the Training Dynamics of Diffusion Models}, 
      author={Tero Karras and Miika Aittala and Jaakko Lehtinen and Janne Hellsten and Timo Aila and Samuli Laine},
      year={2024},
      eprint={2312.02696},
      archivePrefix={arXiv},
      primaryClass={cs.CV},
      url={https://arxiv.org/abs/2312.02696}, 
}

@misc{boffi2025buildconsistencymodellearning,
      title={How to build a consistency model: Learning flow maps via self-distillation}, 
      author={Nicholas M. Boffi and Michael S. Albergo and Eric Vanden-Eijnden},
      year={2025},
      eprint={2505.18825},
      archivePrefix={arXiv},
      primaryClass={cs.LG},
      url={https://arxiv.org/abs/2505.18825}, 
}

@InProceedings{lee2023minimizing,
  title = 	 {Minimizing Trajectory Curvature of {ODE}-based Generative Models},
  author =       {Lee, Sangyun and Kim, Beomsu and Ye, Jong Chul},
  booktitle = 	 {Proceedings of the 40th International Conference on Machine Learning},
  pages = 	 {18957--18973},
  year = 	 {2023},
  editor = 	 {Krause, Andreas and Brunskill, Emma and Cho, Kyunghyun and Engelhardt, Barbara and Sabato, Sivan and Scarlett, Jonathan},
  volume = 	 {202},
  series = 	 {Proceedings of Machine Learning Research},
  month = 	 {23--29 Jul},
  publisher =    {PMLR},
  url = 	 {https://proceedings.mlr.press/v202/lee23j.html},
}

@misc{sabour2024alignstepsoptimizingsampling,
      title={Align Your Steps: Optimizing Sampling Schedules in Diffusion Models}, 
      author={Amirmojtaba Sabour and Sanja Fidler and Karsten Kreis},
      year={2024},
      eprint={2404.14507},
      archivePrefix={arXiv},
      primaryClass={cs.CV},
      url={https://arxiv.org/abs/2404.14507}, 
}

@misc{gupta2026sharpenflowsharpnessawaresampling,
      title={Sharpen Your Flow: Sharpness-Aware Sampling for Flow Matching}, 
      author={Aditi Gupta and Soon Hoe Lim and Annan Yu and N. Benjamin Erichson},
      year={2026},
      eprint={2605.11547},
      archivePrefix={arXiv},
      primaryClass={cs.LG},
      url={https://arxiv.org/abs/2605.11547}, 
}

@misc{chen2024trajectoryregularityodebaseddiffusion,
      title={On the Trajectory Regularity of ODE-based Diffusion Sampling}, 
      author={Defang Chen and Zhenyu Zhou and Can Wang and Chunhua Shen and Siwei Lyu},
      year={2024},
      eprint={2405.11326},
      archivePrefix={arXiv},
      primaryClass={cs.LG},
      url={https://arxiv.org/abs/2405.11326}, 
}

@misc{khan2026isokineticflowmatchingpathwise,
      title={Isokinetic Flow Matching for Pathwise Straightening of Generative Flows}, 
      author={Tauhid Khan},
      year={2026},
      eprint={2604.04491},
      archivePrefix={arXiv},
      primaryClass={cs.LG},
      url={https://arxiv.org/abs/2604.04491}, 
}

@inproceedings{sabour2025alignyourflow,
  title     = {Align Your Flow:
               Scaling Continuous-Time Flow Map Distillation},
  author    = {Sabour, Amirmojtaba and Fidler, Sanja and Kreis, Karsten},
  booktitle = {Advances in Neural Information Processing Systems},
  volume    = {38},
  year      = {2025},
  url       = {https://arxiv.org/abs/2506.14603}
}

@misc{shaul2026paralleldecoding,
  title         = {Parallel Decoding Distillation
                   for Fast Image and Video Generation},
  author        = {Shaul, Neta and Liu, Chao
                   and Vahdat, Arash and Berner, Julius},
  year          = {2026},
  eprint        = {2607.26004},
  archivePrefix = {arXiv},
  url           = {https://arxiv.org/abs/2607.26004}
}

@misc{zhao2026meanflowdistillation,
  title         = {Mean Flow Distillation:
                   Robust and Stable Distillation for Flow Matching Models},
  author        = {Zhao, An and Zhang, Shengyuan and Sun, Zhongjian
                   and Zhou, Yixiang and Li, Zejian and Yang, Ling
                   and Chen, Tianrun and Sun, Lingyun},
  year          = {2026},
  eprint        = {2606.11155},
  archivePrefix = {arXiv},
  url           = {https://arxiv.org/abs/2606.11155}
}

@inproceedings{stancevic2025entropic,
  title     = {Entropic Time Schedulers for Generative Diffusion Models},
  author    = {Stan{\v c}evi{\'c}, Dejan and Handke, Florian
               and Ambrogioni, Luca},
  booktitle = {Advances in Neural Information Processing Systems},
  volume    = {38},
  year      = {2025},
  url       = {https://arxiv.org/abs/2504.13612}
}

@misc{issachar2026complexitybalanced,
  title         = {Complexity-Balanced Diffusion Splitting},
  author        = {Issachar, Noam and Lischinski, Dani and Fattal, Raanan},
  year          = {2026},
  eprint        = {2606.06477},
  archivePrefix = {arXiv},
  url           = {https://arxiv.org/abs/2606.06477}
}

@inproceedings{tsimpos2025optimal,
  title     = {Optimal Scheduling of Dynamic Transport},
  author    = {Tsimpos, Panos and Ren, Zhi
               and Zech, Jakob and Marzouk, Youssef},
  booktitle = {Proceedings of the Thirty-Eighth Conference
               on Learning Theory},
  series    = {Proceedings of Machine Learning Research},
  volume    = {291},
  pages     = {5441--5505},
  publisher = {PMLR},
  year      = {2025},
  url       = {https://proceedings.mlr.press/v291/tsimpos25a.html}
}

@misc{chen2025lipschitz,
  title         = {Lipschitz-Guided Design of Interpolation Schedules
                   in Generative Models},
  author        = {Chen, Yifan and Vanden-Eijnden, Eric and Xu, Jiawei},
  year          = {2025},
  eprint        = {2509.01629},
  archivePrefix = {arXiv},
  url           = {https://arxiv.org/abs/2509.01629}
}

@misc{bondar2026velocityscheduled,
  title         = {Velocity Scheduled Flow Matching},
  author        = {Bondar, Vitalii},
  year          = {2026},
  eprint        = {2607.11442},
  archivePrefix = {arXiv},
  url           = {https://arxiv.org/abs/2607.11442}
}

@misc{tania2026difficultycalibrated,
  title         = {Difficulty-Calibrated Interpolation Paths
                   for Conditional Flow Matching},
  author        = {Tania, Airin Akter and Khan, Md Raihan},
  year          = {2026},
  eprint        = {2608.21286},
  archivePrefix = {arXiv},
  url           = {https://arxiv.org/abs/2608.21286}
}

@misc{lee2026flowmaplanguage,
  title         = {Flow Map Language Models:
                   One-step Language Modeling via Continuous Denoising},
  author        = {Lee, Chanhyuk and Yoo, Jaehoon and Agarwal, Manan
                   and Shah, Sheel and Huang, Jerry
                   and Raghunathan, Aditi and Hong, Seunghoon
                   and Boffi, Nicholas M. and Kim, Jinwoo},
  year          = {2026},
  eprint        = {2602.16813},
  archivePrefix = {arXiv},
  url           = {https://arxiv.org/abs/2602.16813}
}

@article{doi:10.1137/0905021,
author = {Fritsch, F. N. and Butland, J.},
title = {A Method for Constructing Local Monotone Piecewise Cubic Interpolants},
journal = {SIAM Journal on Scientific and Statistical Computing},
volume = {5},
number = {2},
pages = {300-304},
year = {1984},
doi = {10.1137/0905021},
URL = {https://doi.org/10.1137/0905021},
eprint = {https://doi.org/10.1137/0905021},
}

@techreport{krizhevsky2009cifar,
  title={Learning multiple layers of features from tiny images},
  author={Krizhevsky, Alex and Hinton, Geoffrey and others},
  year={2009},
  institution={Technical report, University of Toronto}
}

@inproceedings{liu2015celeba,
  title     = {Deep Learning Face Attributes in the Wild},
  author    = {Liu, Ziwei and Luo, Ping and Wang, Xiaogang and Tang, Xiaoou},
  booktitle = {Proceedings of International Conference on Computer Vision (ICCV)},
  month     = {December},
  year      = {2015}
}

@misc{loshchilov2019decoupledweightdecayregularization,
      title={Decoupled Weight Decay Regularization}, 
      author={Ilya Loshchilov and Frank Hutter},
      year={2019},
      eprint={1711.05101},
      archivePrefix={arXiv},
      primaryClass={cs.LG},
      url={https://arxiv.org/abs/1711.05101}, 
}

@misc{bińkowski2021demystifyingmmdgans,
      title={Demystifying MMD GANs}, 
      author={Mikołaj Bińkowski and Danica J. Sutherland and Michael Arbel and Arthur Gretton},
      year={2021},
      eprint={1801.01401},
      archivePrefix={arXiv},
      primaryClass={stat.ML},
      url={https://arxiv.org/abs/1801.01401}, 
}

@misc{heusel2018ganstrainedtimescaleupdate,
      title={GANs Trained by a Two Time-Scale Update Rule Converge to a Local Nash Equilibrium}, 
      author={Martin Heusel and Hubert Ramsauer and Thomas Unterthiner and Bernhard Nessler and Sepp Hochreiter},
      year={2018},
      eprint={1706.08500},
      archivePrefix={arXiv},
      primaryClass={cs.LG},
      url={https://arxiv.org/abs/1706.08500}, 
}
